\documentclass[11pt]{article}

\usepackage[letterpaper,margin=1in]{geometry}
\usepackage[affil-it]{authblk}
\usepackage{cite}
\usepackage[unicode]{hyperref}
\usepackage[utf8]{inputenc}
\usepackage[english]{babel}
\usepackage{csquotes}
\usepackage{amsmath,amssymb,amsthm}
\usepackage{enumerate}
\usepackage{commath}
\usepackage{bm}
\usepackage{chngcntr}
\usepackage{apptools}

\theoremstyle{definition}
\newtheorem{definition}{Definition}%[section]
\newtheorem{example}{Example}%[section]

\theoremstyle{plain}
\newtheorem{theorem}{Theorem}%[section]
\newtheorem{lemma}{Lemma}%[section]
\newtheorem{proposition}{Proposition}%[section]
\newtheorem{corollary}{Corollary}%[section]

\AtAppendix{\counterwithin{theorem}{section}}
\AtAppendix{\counterwithin{proposition}{section}}
\AtAppendix{\counterwithin{lemma}{section}}
\AtAppendix{\counterwithin{corollary}{section}}
\AtAppendix{\counterwithin{definition}{section}}
\AtAppendix{\counterwithin{example}{section}}

\newcommand{\Comment}[1]{}

\newcommand{\Nats}{\mathbb{N}}

\newcommand{\Reals}{\mathbb{R}}

\newcommand{\A}[1]{\lvert #1 \rvert}
\newcommand{\N}[1]{\lVert #1 \rVert}

\newcommand{\Argmax}[1]{\underset{#1}{\operatorname{arg\,max}}\,}

\newcommand{\Sq}[2]{\{#1\}_{#2 \in \Nats}}
\newcommand{\Sqn}[1]{\Sq{#1}{n}}

\newcommand{\B}{\operatorname{B}}

\DeclareMathOperator{\E}{E}

\newcommand{\M}{\xrightarrow{\textnormal{k}}}
\DeclareMathOperator{\Gr}{graph}

\newcommand{\PM}{\mathcal{P}}
\newcommand{\Lp}{{\operatorname{Lip}}}

\DeclareMathOperator{\Sp}{supp}

\newcommand{\DTV}{\operatorname{d}_{\textnormal{TV}}}
\newcommand{\DKR}{\operatorname{d}_{\textnormal{KR}}}

\newcommand{\Ob}{\mathcal{O}}
\newcommand{\OO}{\Ob^\omega}
\newcommand{\PO}{\pi^\Ob}
\newcommand{\PMO}{\PM(\OO)}
\newcommand{\MC}{\mathcal{H}}
\newcommand{\Gm}{\mathcal{B}}
\newcommand{\GMO}{\Gm(\OO)}
\newcommand{\CO}{C(\OO)}
\newcommand{\GC}{\mathfrak{G}}
\DeclareMathOperator{\V}{V}
\DeclareMathOperator{\SV}{\Sigma V}
\DeclareMathOperator{\SVM}{\Sigma V_{\min}}
\DeclareMathOperator{\SVX}{\Sigma V_{\max}}
\DeclareMathOperator{\Ab}{A}
\DeclareMathOperator{\Nr}{N}
\newcommand{\Bd}{\Lambda}
\DeclareMathOperator{\PG}{\Gamma}
\newcommand{\BM}{\bm{\mu}}
\newcommand{\F}{\mathcal{F}}

\begin{document}

%+Title
\title{Forecasting using incomplete models}
\author{Vanessa Kosoy\footnote{E-mail: \href{mailto:vanessa.kosoy@intelligence.org}{vanessa.kosoy@intelligence.org}.} }
\date{}%\date{\today}
\maketitle
%-Title

%+Abstract
\begin{abstract}
We consider the task of forecasting an infinite sequence of future observations based on some number of past observations, where the probability measure generating the observations is \enquote{suspected} to satisfy one or more of a set of \emph{incomplete} models, i.e., convex sets in the space of probability measures. This setting is in some sense intermediate between the \emph{realizable} setting where the probability measure comes from some known set of probability measures (which can be addressed using, e.g., Bayesian inference) and the \emph{unrealizable} setting where the probability measure is completely arbitrary. We demonstrate a method of forecasting which guarantees that, whenever the true probability measure satisfies an incomplete model in a given countable set, the forecast converges to the same incomplete model in the (appropriately normalized) Kantorovich-Rubinstein metric. This is analogous to merging of opinions for Bayesian inference, except that convergence in the Kantorovich-Rubinstein metric is weaker than convergence in total variation.
\end{abstract}
%-Abstract

\textbf{keywords:} statistical learning theory, online learning, sequence prediction, probability theory, Knightian uncertainty

%+Contents
%\tableofcontents
%-Contents

\section{Introduction}

Forecasting future observations based on past observations is one of the fundamental problems in machine learning, and more broadly is one of the fundamental components of rational reasoning in general. This problem received a great deal of attention, using different methods (see e.g. \cite{Cesa-Bianchi_2006} or Chapter 21 in \cite{Shalev-Shwartz_2014}). Most of those methods assume fixing a class $\MC$ of models or \enquote{hypotheses}, each of which defines a probability measure on sequences of observations (and also conditional probability measures), and produce a forecast $F^\MC$ which satisfies at least one of two kinds of guarantees:

\begin{itemize}
\item 
In the \emph{realizable} setting, the guarantee is that if the observations are sampled from some $\mu \in \MC$, then $F^\MC$ converges to an \enquote{ideal} forecast in some sense.
\item
In the \emph{unrealizable} setting, the guarantee is that for \emph{any} sequence of observations, $F^\MC$ is asymptotically as good as the forecast produced by any $\mu \in \MC$.
\end{itemize}

The realizable setting is often unrealistic, in particular because it requires that the environment under observation is simpler than the observer itself. Indeed, even though we avoid analyzing computational complexity in this work, it should be noted that the computational (e.g., time) complexity of a forecaster is always greater than the complexities of all $\mu \in \MC$. On the other hand, the unrealizable setting usually only provides guarantees for short-term forecasts (since otherwise the training data is insufficient). The latter is in contrast to, e.g., Bayesian inference where the time to learn the model depends on its prior probability, but once \enquote{learned} (i.e., once $F^\MC$ has converged to a given total variation distance from the true probability measure), arbitrarily long-term forecasts become reliable.

The spirit of our approach is that the environment might be very complex, but at the same time it might possess some simple features, and it is these features that the forecast must capture. For example, if we consider a sequence of observations $\{o_n \in \{0,1\}\}_{n \in \Nats}$ s.t. $o_{2k+1}=o_{2k}$, then whatever is the behavior of the even observations $o_{2k}$, the property $o_{2k+1}=o_{2k}$ should asymptotically be assigned high probability by the forecast (this idea was discussed in \cite{Hutter_2009} as \enquote{open problem 4j}).

Formally, we introduce the notion of an \emph{incomplete model}, which is a convex set $M$ in the space $\PMO$ of probability measures on the space of sequences $\OO$. Such an incomplete model may be regarded as a hybrid of probabilistic and \emph{Knightian} uncertainty. We then consider a countable\footnote{This work only deals with the \emph{nonparametric} setting in which $\MC$ is discrete.} set $\MC$ of incomplete models. For any $M \in \MC$ and $\mu \in M$, our forecasts will converge to $M$ in an appropriate sense with $\mu$-probability 1. This convergence theorem can be regarded as an analogue for incomplete models of Bayesian merging of opinions (see \cite{Blackwell_1962}), and is our main result. Our setting can be considered to be in between realizable and unrealizable: it is \enquote{partially realizable} since we require the environment to conform to some $M \in \MC$, it is \enquote{partially unrealizable} since $\mu \in M$ can be chosen adversarially (we can even allow non-oblivious choice, i.e., dependence on the forecast itself).

The forecasting method we demonstrate is based on the principles introduced in \cite{Garrabrant_2016} (similar ideas appeared in \cite{Vovk_2005} in a somewhat simpler setting). The forecast may be regarded as the pricing of a certain combinatorial prediction market, with a countable set of gamblers making bets. The market pricing is then defined by the requirement that the aggregate of all gamblers doesn't make a net profit. The existence of such a pricing follows from the Kakutani-Glicksberg-Fan fixed point theorem and the Kuratowski-Rill-Nardzewski measurable selection theorem (the latter in order to show that the dependence on observation history can be made measurable). The fact that the aggregate of all gamblers cannot make a net profit implies that each individual gambler can only make a bounded profit.

The above method is fairly general, but for our purposes, we associate with each incomplete model $M \in \MC$ a set of gamblers that make bets which are guaranteed to be profitable assuming the true environment $\mu$ satisfies the incomplete model (i.e. $\mu \in M$). The existence of these gamblers follows from the Hahn-Banach separation theorem (where the incomplete model defines the convex set in question) and selection theorems used to ensure that the choice of separating functional depends in a sufficiently \enquote{regular} way on the market pricing. In order to show that the forecaster we described satisfies the desired convergence property, we use tools from martingale theory.

The structure of the paper is as follows. Section~\ref{sec:learning} defines the setting and states the main theorem. Appendix~\ref{sec:garrabrant} lays out the formalism of \enquote{combinatorial prediction markets,} and the central concept of a \emph{dominant forecaster}. Appendix~\ref{sec:prudent} introduces the concept of a \emph{prudent gambling strategy}, which is a technique for proving convergence theorems about dominant forecasters. Appendix~\ref{sec:construction} describes the gamblers associated with incomplete models and completes the proof of the main theorem. Appendix~\ref{sec:theorems} lists two theorems by other authors that we use. Appendix~\ref{sec:examples} proves the validity of some examples given in Section~\ref{sec:learning}.

\section{Results}%{Learning Incomplete Models}
\label{sec:learning}

$\Nats$ will denote the set of natural numbers $\{0, 1, 2 \ldots\}$.

Let $\Sqn{\Ob_n}$ be a sequence of compact Polish spaces. $\Ob_n$ represents the space of possible observations at time $n$. Denote $\Ob^n := \prod_{m < n} \Ob_m$  and $\Ob^\omega:=\prod_{n \in \Nats} \Ob_n$. $\PO_n: \Ob^\omega \rightarrow \Ob^n$ is the projection mapping and $x_{:n}:=\PO_n\left(x\right)$. Given $y \in \Ob^n$, $y\OO := \left(\PO_n\right)^{-1}\left(y\right)$ is a closed subspace of $\OO$. Given $A\subseteq\Ob^n$, we denote $A\Ob^\omega:=\left(\PO_n\right)^{-1}(A)$. $\Ob^n$ will be regarded as a topological space using the product topology and as a measurable space with the $\sigma$-algebra of \emph{universally measurable} sets\footnote{The fact we use universally measurable sets rather than Borel sets will play an important role in the proof of Lemma~\ref{lmm:savvy}.}. For any measurable space $X$, $\PM\left(X\right)$ will denote the space of probability measures on $X$. When $X$ is a Polish space with the Borel $\sigma$-algebra, $\PM\left(X\right)$ will be regarded as a topological space using the weak topology and as a measurable space using the $\sigma$-algebra of Borel sets. Given $\mu \in \PM\left(X\right)$, we denote $\Sp \mu \subseteq X$ the support of $\mu$.

\begin{samepage}
\begin{definition}

A \emph{forecaster} $F$ is a family of measurable mappings

\[\Sqn{F_n: \Ob^n \rightarrow \PMO}\]

s.t. $\Sp {F_n\left(y\right)} \subseteq y\OO$.

\end{definition}
\end{samepage}

Given a forecaster $F$ and $y \in \Ob^n$, $F_n\left(y\right)$ represents the forecast corresponding to observation history $y$. The condition $\Sp {F_n\left(y\right)} \subseteq y\OO$ reflects the obvious requirement of consistency with past observations.

\begin{samepage}
\begin{example}

Consider any $\mu\in\PMO$. Then we can take $F_n(y)$ to be a regular conditional probability of $\mu$ (where the condition is $x \in y\OO$). This corresponds to Bayesian forecasting with prior $\mu$.

\end{example}
\end{samepage}

In order to formulate a claim about forecast convergence, we will need a metric on $\PMO$. Consider $\rho: \OO \times \OO \rightarrow \Reals$ a metrization of $\OO$. Let $\Lp\left(\OO,\rho\right)$ be the Banach space of $\rho$-Lipschitz functions on $\OO$, equipped with the norm

\begin{equation}
\N{f}_\rho:=\max_{x}{\A{f\left(x\right)}} + \sup_{x \ne y} \frac{\A{f\left(x\right)-f\left(y\right)}}{\rho\left(x,y\right)}
\end{equation}

$\PMO$ can be regarded as a compact subset of the dual space $\Lp\left(\OO,\rho\right)'$, yielding the following metrization of $\PMO$:

\begin{equation}
\DKR^\rho\left(\mu,\nu\right):=\sup_{\N{f}_\rho \leq 1}{\left(\E_\mu\left[f\right] - \E_\nu\left[f\right]\right)}
\end{equation}

We call $\DKR^\rho$ the \emph{Kantorovich-Rubinstein metric}\footnote{This is slightly different from the conventional definition but strongly equivalent (i.e. each metric is bounded by a constant multiple of the other). Other names used in the literature for the strongly equivalent metric are \enquote{1st Wasserstein metric} and \enquote{earth mover's distance.}}.

Fix any $x \in \OO$. It is easy to see that

\begin{equation}
\lim_{n \rightarrow \infty} \max_{x' \in x_{:n}\OO} \rho\left(x', x\right) = 0
\end{equation}

Denote $\delta_x \in \PMO$ the unique probability measure s.t. $\delta_x\left(\{x\}\right)=1$. It follows that for any sequence $\Sqn{\mu_n \in \PMO}$ s.t. $\Sp{\mu_n} \subseteq x_{:n}\OO$

\begin{equation}
\lim_{n \rightarrow \infty} \DKR^\rho\left(\mu_n, \delta_x\right) = 0
\end{equation}

Therefore, formulating a non-vacuous convergence theorem requires \enquote{renormalizing} $\DKR$ for each $n \in \Nats$. To this end, we consider a \emph{sequence} of metrizations of $\OO$: $\Sqn{\rho_n: \OO \times \OO \rightarrow \Reals}$. We denote $\DKR^n:=\DKR^{\rho_n}$. In the special case when $\Ob_n=\Ob$ for all $n \in \Nats$ and some fixed space $\Ob$, there is a natural class of metrizations with the property that $\rho_n(yx,yx')=\rho(x,x')$ for some fixed metric $\rho$ on $\OO$ and any $y \in \Ob^n$, $x,x' \in \OO$\footnote{In this special case, the need for renormalization is a side effect of our choice of notation where the forecaster produces a probability measure over the entire sequence rather than over future observations only. Also, in this case the choice of $\rho$ determines everything, since we only use Kantorovich-Rubinstein distance between measures with support in $y\OO$ for the same $y \in \Ob^n$.}. However, in general we can choose any sequence.

Another ingredient we will need is a notion of regular conditional probability for incomplete models. Given measurable spaces $X$ and $Y$, we will use the notation $K: X \M Y$ to denote a Markov kernel with source $X$ and target $Y$. Given $x \in X$, we will use the notation $K\left(x\right) \in \PM\left(Y\right)$. Given $\mu \in \PM\left(X\right)$, $\mu \ltimes K \in \PM\left(X \times Y\right)$ denotes the semidirect product of $\mu$ and $K$, that is, the unique measure satisfying

\begin{equation}
(\mu \ltimes K)(A \times B) = \int_A K(x)(B)\, \mu(dx)
\end{equation} 

Given $\pi:X\rightarrow Y$ measurable, $\pi_*\mu\in\PM(Y)$ denotes the pushforward of $\mu$ by $\pi$. $K_* \mu \in \PM\left(Y\right)$ denotes the pushforward of $\mu$ by $K$ (i.e. the pushforward of $\mu \ltimes K$ by the projection to $Y$). Of course $\pi$ can be regarded as a \enquote{deterministic} Markov kernel, so the notation $\pi_*\mu$ is truly a special case of $K_*\mu$. When $X,Y$ are Polish and $\pi: X \rightarrow Y$ is Borel measurable, $\mu \mid \pi: Y \M X$ is defined to be s.t. $\pi_* \mu \ltimes \left(\mu \mid \pi\right)$ is supported on the graph of $\pi$ and $\left(\mu \mid \pi\right)_* \pi_* \mu = \mu$ (i.e. $\mu \mid \pi$ is a regular conditional probability). By the disintegration theorem, $\mu \mid \pi$ exists and is defined up to coincidence $\pi_* \mu$-almost everywhere.

\begin{samepage}
\begin{definition}
\label{def:update_incomplete}

Let $X,Y$ be compact Polish spaces, $\pi: X \rightarrow Y$ continuous and $M \subseteq \PM(X)$. We say that $N: Y \rightarrow 2^{\PM(X)}$ is a \emph{regular upper bound for $M \mid \pi$} when 
\begin{enumerate}[i.]
\item\label{con:def__update_incomplete__clos} The set $\Gr{N}:=\{(y,\mu)\in Y\times\PM(X) \mid \mu \in N(y)\}$ is closed.
\item\label{con:def__update_incomplete__conv} For every $y \in \Ob^n$, $N(y)$ is convex.
\item\label{con:def__update_incomplete__cond} For every $\mu \in M$ and $\pi_*\mu$-almost every $y \in Y$, $(\mu \mid \pi)(y) \in N(y)$.
\end{enumerate}

\end{definition}
\end{samepage}

The reason we call that \enquote{regular upper bound for $M\mid\pi$} rather than just \enquote{$M\mid\pi$} is that, roughly, our conditions guarantee $N$ is \enquote{big enough} but don't guarantee it is not \enquote{too big}. For example, setting $N(y):=\PM(X)$ would trivially satisfy all conditions. Also note that, although technically we haven't assumed $M$ is convex, we might as well have assumed it: it is not hard to see that, if $N$ is a regular upper bound for $M \mid \pi$ and $M'$ is the convex hull of $M$, then $N$ is a regular upper bound for $M' \mid \pi$. The same remark applies to Theorem~\ref{thm:main} below.

We give a few examples for Definition~\ref{def:update_incomplete}. The proofs that these examples are valid are in Appendix~\ref{sec:examples}.

\begin{samepage}
\begin{example}
\label{exm:update_incomplete_finite}

Suppose that $M$ is convex and $Y$ is a finite set (in particular, this example is applicable when the $\Ob_n$ are finite sets, $X=\OO$, $Y=\Ob^n$ and $\pi=\PO_n$). Then, there is a unique $N: Y \rightarrow 2^{\PM(X)}$ which is a \emph{minimal} (w.r.t. set inclusion) regular upper bound for $M \mid \pi$ and we have

\begin{equation}
\label{eqn:exm__update_incomplete_finite}
N(y) = \overline{\{\left(\mu \mid \pi^{-1}(y)\right) \mid \mu \in M,\, \mu\left(\pi^{-1}(y)\right) > 0\}}
\end{equation}

Here, $\mu \mid \pi^{-1}(y)$ is the conditional probability measure and the overline stands for topological closure.

\end{example}
\end{samepage}
\begin{samepage}
\begin{example}
\label{exm:update_incomplete_kernels}

Consider some $I \subseteq \Nats$ and suppose that for each $n \in I$, we are given $K_n: \Ob^n \M \Ob_n$. Assume that each $K_n$ is \emph{Feller continuous}, that is, that the induced mapping $\Ob^n \rightarrow \PM\left(\Ob_n\right)$ is continuous. Define $M^K$ by

\begin{equation}
M^K:=\{\mu \in \PMO \mid \forall n \in I: \PO_{n+1*}\mu = \PO_{n*}\mu \ltimes K_n\}
\end{equation}

For each $n \in \Nats$, define $M^K_n: \Ob^n \rightarrow 2^{\PMO}$ by

\begin{equation}
M^K_n(y) = \{\mu \in \PMO \mid \Sp{\mu} \subseteq y\OO,\, \forall m \in I: m \geq n \implies \PO_{m+1*}\mu = \PO_{m*}\mu \ltimes K_m\}
\end{equation}

Then, $M^K_n$ is a regular upper bound for $M^K \mid \PO_n$.

\end{example}
\end{samepage}

\begin{samepage}
\begin{example}
\label{exm:update_incomplete_euclid}

Suppose that $Y$ is a compact subset of $\Reals^{d}$ for some $d \in \Nats$ (in particular, this example is applicable when each $\Ob_n$ is a compact subset of $\Reals^{d_n}$, $X=\OO$, $Y=\Ob^n$ and $\pi=\PO_n$; in that case $d=\sum_{m < n} d_n$). We regard $Y$ as a metric space using the Euclidean metric. For any $y \in Y$ and $r > 0$, $\B_r\left(y\right)$ will denote the open ball of radius $r$ with center at $y$. 

Fix any $M \subseteq \PM(X)$. Then, there is a unique $N: Y \rightarrow 2^{\PM(X)}$ which is minimal among mappings which both satisfy conditions \ref{con:def__update_incomplete__clos} and \ref{con:def__update_incomplete__conv} of Definition~\ref{def:update_incomplete} and are s.t. for any $y \in Y$, $\nu \in \PM(X)$ and $\mu \in M$, if $y \in \Sp \pi_*\mu$ and $\nu = \lim_{r \rightarrow 0}{\left(\mu \mid \pi^{-1}\left(\B_r\left(y\right)\right)\right)}$, then $\nu \in N(y)$ (the limit is defined using the the weak topology). Moreover, $N$ is a regular upper bound for $M \mid \pi$.

\end{example}
\end{samepage}

We are now ready to formulate the main result.

Given a metric space $X$ with metric $\rho: X \times X \rightarrow \Reals$, $x \in X$ and $A \subseteq X$, we will use the notation

\begin{equation}
\rho\left(x,A\right):=\inf_{y \in A} \rho\left(x,y\right)
\end{equation}

\begin{theorem}
\label{thm:main}

Fix any $\MC \subseteq 2^{\PMO}$ countable. For every $M \in \MC$ and $n \in \Nats$, let $M_n$ be a regular upper bound for $M \mid \PO_n$. Then, there exists a forecaster $F^\MC$ s.t. for any $M \in \MC$, $\mu \in M$ and $\mu$-almost any $x \in \OO$

\begin{equation}
\label{eqn:thm_main}
\lim_{n \rightarrow \infty} \DKR^n\left(F^\MC_n\left(x_{:n}\right),M_n\left(x_{:n}\right)\right) = 0
\end{equation}

\end{theorem}

That is, the forecaster $F^\MC$ ensures that, for any incomplete model $M$ satisfied by the true environment $\mu$, the forecast will (almost surely) converge to the model (as opposed to complete models, there may be several incomplete models satisfied by the same environment).

Note that $F^\MC$ as above exists for any choice of $\Sqn{\rho_n}$ but it depends on the choice. Informally, we can think of this choice as determining the duration of the future time period over which we need our forecast to be reliable. It also might depend on the choice of regular upper bounds (of course, if there are \emph{minimal} regular upper bounds, these yield a forecaster than works for any other regular upper bounds).

The example from the Introduction section can now be realized as follows. Take $\Ob_n=\Ob=\{0,1\}$ and let $M^K \subseteq \PMO$ be as in Example~\ref{exm:update_incomplete_kernels}, where $I = 2\Nats+1$ and $\Pr_{o\sim K_{2n+1}(y)}\left[o=y_{2n}\right]=1$. Define $\rho_n$  by

\begin{equation}
\rho_n(x,x'):=\begin{cases} \max\{2^{n-m} \mid x_{m} \ne x'_{m}\} \text{ if } x \ne x' \\ 0 \text{ if } x=x'\end{cases}
\end{equation}

If $M^K \in \MC$, and $x\in\OO$ is s.t. $x_{2n+1}=x_{2n}$, then for any $k \in \Nats$, the forecaster $F^\MC$ of Theorem~\ref{thm:main} satisfies

\begin{align}
\lim_{n\rightarrow\infty} \Pr_{x' \sim F_{2n}\left(x_{:2n}\right)}&\left[x'_{2(n+k)+1}=x'_{2(n+k)}\right] = 1\\
\lim_{n\rightarrow\infty} \Pr_{x' \sim F_{2n+1}\left(x_{:2n+1}\right)}&\left[x'_{2(n+k)+1}=x'_{2(n+k)}\right] = 1
\end{align}

If we only cared about predicting the next observation (as opposed to producing a probability measure over the entire sequence), this example (and any other instance of Example~\ref{exm:update_incomplete_kernels}, at least in the case $\Ob_n=\Ob$ finite) would be a special case of the \enquote{sleeping experts} setting in online learning (see \cite{Freund_1997}). However, our formalism is much more general, even for predicting the next observation. For instance, we can consider $\Ob=\{0,1,2,3\}$ and have an incomplete model specifying that the next observation is an odd number (thus, the \enquote{expert} specializes by predicting specific \emph{properties} of the observation rather than only making predictions at specific times). As another example, we can consider $\Ob=\{0,1,2\}$ and an incomplete model specifying that the probability distribution $\left(p_0,p_1,p_2\right)$ of each observation satisfies $p_0,p_1,p_2 \geq 0.1$. This model would capture any environment that can be regarded as a process observed through a noisy sensor that has a probability of 0.3 to output a uniformly random observation instead of the real state of the process.

The rest of the paper is devoted to proving Theorem~\ref{thm:main}.

\appendix

\section{Appendix: Dominant Forecasters}
\label{sec:garrabrant}

In this section we explain our generalization of the methods introduced in \cite{Garrabrant_2016} under the name \enquote{logical inductors.} The main differences between our formalism and that of \cite{Garrabrant_2016} are

\begin{itemize}
\item 
We are interested in sequence forecasting rather than formal logic.
\item
We consider probability measures on certain Polish spaces, rather than probability assignment functions on finite sets.
\item
In particular, no special treatment of expected values is required.
\item
The observations are stochastic rather than deterministic.
\end{itemize}

Our terminology also differs from \cite{Garrabrant_2016}: their \enquote{market} is our \enquote{forecaster}, their \enquote{trader} is our \enquote{gambler}.

Similar ideas were investigated in \cite{Vovk_2005} under the name \enquote{defensive forecasting}. However in \cite{Vovk_2005}, the forecast is a single probability of an impending observation in $\{0,1\}$ rather than a probability measure in the space of infinite sequences of observations taking values in arbitrary compact Polish spaces.

In any case, our exposition assumes no prior knowledge about logical inductors or defensive forecasting.

The idea is to consider a collection of gamblers making bets against the forecaster. If a gambler with finite budget cannot make an infinite profit, the gambler is said to be \enquote{dominated} by the forecaster. We then prove that for any countable collection of gamblers, there is a forecaster that dominates all of them.

Given a compact Polish space $X$, $C\left(X\right)$ will denote the Banach space of continuous functions with uniform norm. We will also use the shorthand notation $\Gm\left(X\right):=C\left(\PM\left(X\right) \times X\right)$. Equivalently, $\Gm\left(X\right)$ can be regarded to be the space of continuous functions from $\PM\left(X\right)$ to $C\left(X\right)$, and we will curry implicitly in our notation.

We regard $\GMO$ as the space of \emph{bets} that can be made against a forecaster. Given $\beta \in \GMO$, a forecast $\mu \in \PMO$ and observation sequence $x \in \OO$, the payoff of the bet $\beta$ is

\begin{equation}
\V{\beta}\left(\mu,x\right):=\beta\left(\mu,x\right) - \E_{x' \sim \mu}\left[\beta\left(\mu,x'\right)\right]
\end{equation}

This definition ensures that $\E_{x\sim\mu}\left[\V{\beta}(\mu,x)\right] = 0$, so that the bet is fair from the perspective of the forecaster.

Note that this defines a bounded linear operator $\V: \Gm\left(X\right) \rightarrow \Gm\left(X\right)$.

\begin{samepage}
\begin{example}

Suppose that the $\Ob_n$ are finite sets with discrete topology. Fix some $n\in\Nats$ and $A \subseteq \Ob^n$. Define $\beta^A \in \GMO$ by

\begin{equation}
\beta^{A}(\mu,x):=\begin{cases} 1 \text{ if } x_{:n} \in A \\ 0 \text{ if } x_{:n} \not\in A \end{cases}
\end{equation}

$\beta^y$ represents betting that the sequence of observations will begin with an element of $A$. The payoff $\V{\beta^A}(\mu,x)$ is $1 - \mu\left(A\Ob^\omega\right)$ when $x_{:n} \in A$ and $-\mu\left(A\Ob^\omega\right)$ when $x_{:n} \not\in A$.

\end{example}
\end{samepage}

$C\left(X\right)$ and $\Gm\left(X\right)$ will also be regarded as a measurable spaces, using the $\sigma$-algebra of Borel sets. We remind that the $\sigma$-algebra on $\Ob^n$ is the algebra of \emph{universally measurable} sets.

\begin{definition}

A \emph{gambler} is a family of measurable mappings

\[\Sqn{G_n : \Ob^n \times \PMO^n \rightarrow \Gm\left(\OO\right)}\]

\end{definition}

A gambler is considered to observe a forecaster and bet against it. Given $y \in \Ob^n$ and $\BM \in \PMO^n$, $G_n\left(y,\BM\right)=\beta$ means that, if $y$ are the first $n$ observations and $\BM$ are the first $n$ forecasts (made after observing $n-1$ out of the $n$ observations), the gambler will make bet $\beta$ (which is in itself a function of the forecast the forecaster makes after the $n$-th observation).

\begin{samepage}
\begin{example}

Suppose that the $\Ob_n$ are finite sets with discrete topology. Consider a family of mappings $\left\{f_n: \Ob^n \rightarrow \Ob_n\right\}_{n\in\Nats}$. Define the gambler $G^f$ by

\begin{equation}
G^f_n(y,\BM):=\beta^{\{yf(y)\}}
\end{equation}

That is, $G^f$ always bets on the next observation being $f(y)$.

\end{example}
\end{samepage}

We now introduce some notation regarding the interaction of gamblers and forecasters. 

Given a gambler $G$ and a forecaster $F$, we define the measurable mappings $G^F_n: \Ob^n \rightarrow \GMO$ by

\begin{equation}
\label{eqn:GF}
G^F_n\left(y\right):=G_n\left(y,F_0,F_1\left(y_{:1}\right) \ldots F_{n-1}\left(y_{:n-1}\right)\right)
\end{equation}

Here, $y_{:m}$ denotes the projection of $y$ to $\Ob^m$.

We define the measurable mappings $\overline{\V G}^F_n: \Ob^n \rightarrow \CO$ and $\SV G^F_n: \Ob^{n-1} \rightarrow \CO$ by

\begin{equation}
\overline{\V G}^F_n\left(y\right):=\left(\V G^F_n\left(y\right)\right)\left(F_n\left(y\right)\right)
\end{equation}

\begin{equation}
\SV G^F_n\left(y\right) := \sum_{m < n} \overline{\V G}^F_m\left(y_{:m}\right)
\end{equation}

In the definition of $\SV G^F_0$, $\Ob^{-1}$ is considered to be equal to $\Ob^0$ (i.e. the one point space).

That is, $\overline{\V G}^F_n\left(y\right)(x)$ is the payoff of the $n$-th gamble of gambler $G$ playing against forecaster $F$, assuming initial history $y$ and full history $x$. $\SV G^F_n\left(y\right)$ is the total payoff of the first $n$ gambles.

We define the measurable mappings $\SVM G^F_n, \SVX G^F_n: \Ob^{n-1} \rightarrow \Reals$ by

\begin{equation}
\label{eqn:svmf}
\SVM G^F_n\left(y\right):=\min_{y\OO} \SV G^F_n(y)
\end{equation}

\begin{equation}
\label{eqn:svxf}
\SVX G^F_n\left(y\right):=\max_{y\OO} \SV G^F_n(y)
\end{equation}

Thus, $\SVM G^F_n\left(y\right)$ is the minimal possible payoff of the first $n$ gambles and $\SVX G^F_n\left(y\right)$ is the maximal possible payoff. $y$ appears twice on the right hand side of equations (\ref{eqn:svmf},\ref{eqn:svxf}) because, first, it defines the space of histories over which we minimize/maximize and, second, it defined the input to the gambler and forecaster.

We are now ready to state the formally state the definition of \enquote{dominance} alluded to in the beginning of the section.

\begin{definition}
\label{def:dominance}

Consider a forecaster $F$ and a gambler $G$. $F$ is said to \emph{dominate} $G$ when for any $x \in \OO$, if condition~\ref{eqn:def_dominance__loss} holds then condition~\ref{eqn:def_dominance__gain} holds:

\begin{equation}
\label{eqn:def_dominance__loss}
\inf_{n \in \Nats} {\SVM G^F_{n}\left(x_{:n-1}\right)} > -\infty
\end{equation}

\begin{equation}
\label{eqn:def_dominance__gain}
\sup_{n \in \Nats} {\SVX G^F_{n}\left(x_{:n-1}\right)} < +\infty
\end{equation}

\end{definition}

That is, as long as the gambler doesn't go into infinite debt, it cannot make an infinite profit.

\begin{samepage}
\begin{example}

Suppose that the $\Ob_n$ are finite sets with discrete topology. Consider a family of mappings $\left\{f_n: \Ob^n \rightarrow \Ob_n\right\}_{n\in\Nats}$. Define the forecaster $F^f$ by $F_n(y) := \delta_{f_n(y)}$. Then, it is easy to see that $F^f$ dominates $G^f$.

\end{example}
\end{samepage}

The following is the main theorem of this section, which is our analogue of \enquote{Theorem 4.0.6} from \cite{Garrabrant_2016}.

\begin{theorem}
\label{thm:exist_dominant}

Let $\Sq{G^k}{k}$ be a family of gamblers. Then, there exists a forecaster $F^G$ s.t. for any $k \in \Nats$, $F$ dominates $G^k$.

\end{theorem}

The rest of the section is devoted to proving Theorem~\ref{thm:exist_dominant}. The proof will consist of two steps. First we show that for any \emph{single} gambler, there is a dominant forecaster. Then, given a countable set $\GC$ of gamblers, we construct a single gambler s.t. dominating this gambler implies dominating all gamblers in the set.

The following lemma shows that for any bet, there is a forecast which makes the bet unwinnable.

\begin{lemma}
\label{lmm:unwinnable}

Consider $X$ a compact Polish space and $\beta \in \Gm\left(X\right)$. Then, there exists $\mu \in \PM\left(X\right)$ s.t.

\begin{equation}
\Sp \mu \subseteq \Argmax{x\in X} \beta\left(\mu,x\right)
\end{equation}

\end{lemma}

\begin{proof}

Define ${K: \PM\left(X\right) \rightarrow 2^{\PM\left(X\right)}}$ as follows:

\[K\left(\mu\right):=\{\nu \in \PM\left(X\right) \mid \Sp{\nu} \subseteq \Argmax{}{\beta\left(\mu\right)}\}\]

For any ${\mu}$, ${K}\left(\mu\right)$ is obviously convex and non-empty. It is also easy to see that the graph of $K$ in $\PM\left(X\right) \times \PM\left(X\right)$ is closed. Applying the Kakutani-Glicksberg-Fan theorem, we conclude that $K$ has a fixed point, i.e. $\mu \in \PM\left(X\right)$ s.t. $\mu \in K\left(\mu\right)$.
\end{proof}

Note that $\mu$ as above indeed makes $\beta$ unwinnable since

$$\V{\beta}(\mu,x) = \beta\left(\mu,x\right) - \E_{x' \sim \mu}\left[\beta\left(\mu,x'\right)\right] = \beta(\mu,x)- \max_{x'\in X} \beta\left(\mu,x'\right) \leq 0$$

Next, we show that the probability measure of Lemma~\ref{lmm:unwinnable} can be made to depend measurably on past observations and the bet.

\begin{lemma}
\label{lmm:measurable_unwinnable}

Fix $Y_{1,2}$ compact Polish spaces and denote $X:=Y_1 \times Y_2$. Then, there exists a Borel measurable mapping $\alpha: Y_1 \times \Gm\left(X\right) \rightarrow \PM\left(X\right)$ s.t. for any $y \in Y_1$ and $\beta \in \Gm\left(X\right)$

\begin{equation}
\Sp \alpha\left(y,\beta\right) \subseteq \Argmax{y \times Y_2} \beta\left(\alpha\left(y,\beta\right)\right)
\end{equation}

\end{lemma}

\begin{proof}

Define ${Z_1, Z_2, Z \subseteq Y_1 \times \Gm\left(X\right) \times \PM\left(X\right)}$ by

$${Z_1:=\{\left(y,\beta,\mu\right) \in Y_1 \times \Gm\left(X\right) \times \PM\left(X\right) \mid \Sp \mu \subseteq y \times Y_2\}}$$

$${Z_2:=\{\left(y,\beta,\mu\right) \in Y_1 \times \Gm\left(X\right) \times \PM\left(X\right) \mid \E_\mu\left[\beta\left(\mu\right)\right] = \max_{y_2 \in Y_2} \beta\left(\mu,y,y_2\right)\}}$$

$${Z:=Z_1 \cap Z_2 =\{\left(y,\beta,\mu\right) \in Y_1 \times \Gm\left(X\right) \times \PM\left(X\right) \mid \Sp \mu \subseteq \Argmax{y \times Y_2} \beta\left(\mu\right)\}}$$

We can view ${Z}$ as the graph of a \emph{multivalued} mapping from ${Y_1 \times \Gm\left(X\right)}$ to ${\PM\left(X\right)}$ (i.e a mapping from ${Y_1 \times \Gm\left(X\right)}$ to $2^{\PM\left(X\right)}$). We will now show that this multivalued mapping has a \emph{selection}, i.e. a single-valued Borel measurable mapping whose graph is a subset of $Z$. Obviously, the selection is the desired ${\alpha}$.

Fix $\rho$ a metrization of $Y_1$. $Z_1$ is the vanishing locus of the continuous function $\E_{\left(y_1, y_2\right) \sim \mu}\left[\rho\left(y_1,y\right)\right]$. $Z_2$ is the vanishing locus of the continuous function $\E_\mu\left[\beta\left(\mu\right)\right] - \max_{y_2 \in Y_2} \beta\left(\mu,y,y_2\right)$. Therefore $Z_{1,2}$ are closed and so is $Z$. In particular, the fiber ${Z_{y\beta}}$ of ${Z}$ over any ${\left(y,\beta\right) \in Y_1 \times \Gm\left(X\right)}$ is also closed. 

For any ${y \in Y_1}$, ${\beta \in \Gm\left(X\right)}$, define ${i_y: Y_2 \rightarrow X}$ by ${i_y\left(y_2\right):=\left(y,y_2\right)}$ and ${\beta_y \in \Gm\left(Y_{2}\right)}$ by $\beta_y\left(\nu,y'\right):=\beta\left(i_{y*}\nu,y,y'\right)$. Applying Lemma~\ref{lmm:unwinnable} to ${\beta_y}$ we get ${\nu \in \PM\left(Y_2\right)}$ s.t.

$$\Sp \nu \subseteq \Argmax{} \beta_y\left(\nu\right)$$

It follows that ${\left(y,\beta,i_{y*}\nu\right) \in Z}$ and hence ${Z_{y\beta}}$ is non-empty.

Consider any ${U \subseteq \PM\left(X\right)}$ open. Then, ${A_U:=\left(Y_{1} \times \Gm\left(X\right) \times U\right) \cap Z}$ is locally closed, and in particular, it is an ${F_\sigma}$ set. Therefore, the image of ${A_U}$ under the projection to ${Y_{1} \times \Gm\left(X\right)}$ is also ${F_\sigma}$ and in particular Borel. 

Applying the Kuratowski-Rill-Nardzewski measurable selection theorem, we get the desired result ($Z$ is the graph of the multivalued mapping, we established that the values $Z_{y\beta}$ of this mapping are closed and non-empty, the previous paragraph establishes that the mapping is weakly measurable, and the theorem implies it has a selection, which is $\alpha$).
\end{proof}

We define the measurable mappings $\bar{G}^F_n: \Ob^n \rightarrow \CO$ by

\begin{equation}
\bar{G}^F_n\left(y\right) := G^F_n\left(y;F_n\left(y\right)\right)
\end{equation}

Here, the notation $G^F_n\left(y;F_n\left(y\right)\right)$ means $G^F_n(y)\left(F_n\left(y\right)\right)$ ($G^F_n(y)$ is a function from $\PMO$ to $\CO$ which we apply to $F_n\left(y\right)\in\PMO$). We will use the semicolon in a similar way in the rest of paper as well: to indicate applying a function which is the result of another function.

The following is our analogue of \enquote{Theorem 1} from \cite{Vovk_2005}.

\begin{corollary}
\label{crl:dominate_one}

Let $G$ be a gambler. Then, there exists a forecaster $F$ s.t. for all $n \in \Nats$ and $y \in \Ob_n$

\begin{equation}
\Sp F_n\left(y\right) \subseteq \Argmax{y\OO} \bar{G}^F_n
\end{equation}

\end{corollary}

\begin{proof}

For every $n \in \Nats$, let $\alpha_n: \Ob_n \times \GMO \rightarrow \PMO$ be as in Lemma~\ref{lmm:measurable_unwinnable}. Observing that the definition of $G^F_n$ depends only on $F_m$ for $m < n$, we define $F$ recursively as

\[F_n\left(y\right):=\alpha_n\left(y,G^F_n\left(y\right)\right)\]
\end{proof}

In particular, the forecaster $F$ of Corollary~\ref{crl:dominate_one} dominates $G$ since, as easy to see, $\overline{\V G}^F_n \leq 0$ and hence $\SV G^F_n \leq 0$.

We will now introduce a way to transform a gambler in a way that enforces a \enquote{finite spending budget.} Consider a gambler $G$ and fix $b > 0$ (the size of the \enquote{budget}). Define the measurable functions $\SV G_n: \Ob^{n-1} \times \PMO^n  \rightarrow \CO$ by

\begin{equation}
\SV G_n\left(y,\BM\right) := \sum_{m < n} \left(\V G_m\left(y_{:m},\BM_{:m}\right)\right)\left(\BM_m\right)
\end{equation}

Here, $\BM_{:m}$ denotes the projection of $\BM \in \PMO^n$ to $\PMO^m$ that takes components $l < m$ and $\BM_m$ denotes the $m$-th component of $\BM$. Define the measurable functions $\SVM G_n, \SVX G_n: \Ob^{n-1} \times \PMO^n  \rightarrow \Reals$ by

\begin{equation}
\label{eqn:svm}
\SVM G_n\left(y,\BM\right) := \min_{y\OO}{\SV G_n\left(y,\BM\right)}
\end{equation}

\begin{equation}
\label{eqn:svx}
\SVX G_n\left(y,\BM\right) := \max_{y\OO}{\SV G_n\left(y,\BM\right)}
\end{equation}

Equations~(\ref{eqn:svm},\ref{eqn:svx}) are completely analogical to equations~(\ref{eqn:svmf},\ref{eqn:svxf}) except that we define functions of forecast histories instead of considering a particular forecaster.

Define the measurable sets

\begin{equation}
\Ab_b G_n:=\{\left(y,\BM\right) \in \Ob^{n-1} \times \PMO^n \mid \SVM G_n\left(y,\BM\right) > -b\}
\end{equation}

Note that $\Ab_b G_n \times \Ob_n$ is a measurable subset of $\Ob^n \times \PMO^n$ and in particular can be regarded as a measurable space in itself. Define the measurable functions $\Nr_b G_n: \Ab_b G_n \times \Ob_n \rightarrow C\left(\PMO\right)$ by

\begin{equation}
\Nr_b G_n\left(y,\BM;\nu\right):=\max\left(1,\max_{y\OO} \frac{-\left(\V G_n\left(y,\BM\right)\right)\left(\nu\right)}{\SV G_n\left(y_{:n-1},\BM\right)+b}\right)^{-1}
\end{equation}

As before, the semicolon indicates currying i.e. $\Nr_b G_n\left(y,\BM;\nu\right)=\Nr_b G_n\left(y,\BM\right)(\nu)$.

Finally, define the gambler $\Bd_b G$ as follows

\begin{equation}
\Bd_b G_n\left(y,\BM\right):=\begin{cases} 0 \text{ if } \exists m < n: \left(y_{:m},\BM_{:m+1}\right) \not\in \Ab_b G_{m+1} \\ \Nr_b G_n\left(y,\BM\right) \cdot G_n\left(y, \BM\right) \text{ otherwise} \end{cases}
\end{equation}

The next proposition shows that as long as $G$ stays \enquote{within budget $b$,} the operator $\Bd_b$ has no effect. 

\begin{proposition}
\label{prp:b_no_effect}

Consider any gambler $G$, $b > 0$, $n \in \Nats$, $y \in \Ob^{n - 1}$ and $\BM \in \PMO^n$. Assume that for all $m < n$, $\left(y_{:m},\BM_{:m+1}\right) \in \Ab_b G_{m+1}$. Then, for all $m < n$

\begin{equation}
\Bd_b G_m\left(y_{:m},\BM_{:m};\BM_m\right)=G_m\left(y_{:m},\BM_{:m};\BM_m\right)
\end{equation}

\end{proposition}

\begin{proof}

Consider any $m < n$. We have

$$\SVM G_{m+1}\left(y_{:m},\BM_{:m+1}\right) + b > 0$$

$$\forall x \in y_{:m}\OO: \SV G_{m+1}\left(y_{:m},\BM_{:m+1}; x\right) + b > 0$$

$$\forall x \in y_{:m}\OO: \SV G_{m}\left(y_{:m-1},\BM_{:m}; x\right) + \left(\V G_{m}\left(y_{:m},\BM_{:m}\right)\right)\left(\BM_{m}, x\right) + b > 0$$

$$\forall x \in y_{:m}\OO: \SV G_{m}\left(y_{:m-1},\BM_{:m}; x\right) + b > -\left(\V G_{m}\left(y_{:m},\BM_{:m}\right)\right)\left(\BM_{m}, x\right)$$

Using the assumption again, the left hand side is positive. It follows that

$$\forall x \in y_{:m}\OO: 1 > \frac{-\left(\V G_{m}\left(y_{:m},\BM_{:m}\right)\right)\left(\BM_{m}, x\right)}{\SV G_{m}\left(y_{:m-1},\BM_{:m}; x\right) + b}$$

$$\Nr_b G_m\left(y_{:m},\BM_{:m};\BM_m\right) = 1$$

Since we are in the second case in the definition of ${\Bd_b G_m\left(y_{:m},\BM_{:m}\right)}$, we get the desired result.
\end{proof}

Now, we show that $\Bd_b G$ indeed never \enquote{goes over budget.}

\begin{proposition}
\label{prp:b_stays_in_budget}

Consider any gambler $G$ and $b > 0$. Then, for any $n \in \Nats$, $y \in \Ob^{n-1}$ and $\BM \in \PMO^n$

\begin{equation}
\SVM \Bd_b G_n\left(y,\BM\right) \geq -b
\end{equation} 

\end{proposition}

\begin{proof}

Let ${m_0 \in \Nats}$ be the largest number s.t. ${m_0 \leq n}$ and 

$$\forall m \leq m_0: \SVM G_m\left(y_{:m-1},\BM_{:m}\right) > -b$$

For any ${m \leq m_0}$, we have, by Proposition~\ref{prp:b_no_effect}

$$\SVM \Bd_b G_m\left(y_{:m-1},\BM_{:m}\right)=\SVM G_m\left(y_{:m-1},\BM_{:m}\right) > -b$$

(Note that the sum in the definition of ${\SVM \Bd_b G_m}$ only involves ${\Bd_b G_l}$ for ${l < m \leq m_0}$)

For ${m=m_0+1}$ and any $x \in y_{:m_0}\OO$, we have

$$\SV \Bd_b G_{m_0+1}\left(y_{:m_0},\BM_{:m_0+1};x\right) = \SV \Bd_b G_{m_0}\left(y_{:m_0-1},\BM_{:m_0};x\right) + \left(\V \Bd_b G_{m_0}\left(y_{:m_0},\BM_{:m_0}\right)\right)\left(\BM_{m_0},x\right)$$

The first term only involves ${\Bd_b G_l}$ for ${l < m_0}$, allowing us to apply Proposition~\ref{prp:b_no_effect}. The second term is in the second case of the definition of $\Bd_b$.

\begin{align*}
\SV \Bd_b G_{m_0+1}\left(y_{:m_0},\BM_{:m_0+1};x\right) = &\SV G_{m_0}\left(y_{:m_0-1},\BM_{:m_0};x\right) +\\ &\Nr_b G_{m_0}\left(y_{:m_0},\BM_{:m_0};\BM_{m_0}\right) \cdot \left(\V G_{m_0}\left(y_{:m_0},\BM_{:m_0}\right)\right)\left(\BM_{m_0},x\right)
\end{align*}

If ${x}$ is s.t. ${\left(\V G_{m_0}\left(y_{:m_0},\BM_{:m_0}\right)\right)\left(\BM_{m_0},x\right) \geq 0}$, then

\[\SV \Bd_b G_{m_0+1}\left(y_{:m_0},\BM_{:m_0+1};x\right) \geq \SV G_{m_0}\left(y_{:m_0-1},\BM_{:m_0};x\right)\geq \SVM G_{m_0}\left(y_{:m_0-1},\BM_{:m_0}\right) > -b\]

If ${x}$ is s.t. ${\left(\V G_{m_0}\left(y_{:m_0},\BM_{:m_0}\right)\right)\left(\BM_{m_0},x\right) < 0}$, then, by the definition of $\Nr_b$

\begin{align*}
\SV \Bd_b G_{m_0+1}\left(y_{:m_0},\BM_{:m_0+1};x\right) \geq &\SV G_{m_0}\left(y_{:m_0-1},\BM_{:m_0};x\right)+ \\
&\frac{\SV G_{m_0}\left(y_{:m_0-1},\BM_{:m_0};x\right) + b}{-\left(\V G_{m_0}\left(y_{:m_0},\BM_{:m_0}\right)\right)\left(\BM_{m_0},x\right)} \cdot \left(\V G_{m_0}\left(y_{:m_0},\BM_{:m_0}\right)\right)\left(\BM_{m_0},x\right)
\end{align*}

$$\SV \Bd_b G_{m_0+1}\left(y_{:m_0},\BM_{:m_0+1};x\right) \geq -b$$

Finally, consider ${m > m_0 + 1}$.

$$\SV \Bd_b G_{m}\left(y_{:m-1},\BM_{:m};x\right) = \SV \Bd_b G_{m-1}\left(y_{:m-2},\BM_{:m-1};x\right) + \left(\V \Bd_b G_{m-1}\left(y_{:m-1},\BM_{:m-1}\right)\right)\left(\BM_{m-1},x\right)$$

The second term is in the first case of the definition of $\Bd_b$ and therefore vanishes.

$$\SV \Bd_b G_{m}\left(y_{:m-1},\BM_{:m};x\right) = \SV \Bd_b G_{m-1}\left(y_{:m-2},\BM_{:m-1};x\right)$$

By induction on ${m}$, we conclude:

$$\SVM \Bd_b G_{m}\left(y_{:m-1},\BM_{:m}\right) \geq \SVM \Bd_b G_{m-1}\left(y_{:m-2},\BM_{:m-1}\right) \geq -b$$
\end{proof}

Next, we show that for any gambler there is \enquote{limited budget} gambler s.t. any forecaster that dominates it also dominates the original gambler. For any $x \in \Reals$, we denote $x_+:=\max(x,0)$.

\begin{proposition}
\label{prp:frugal_gambler}

Let $G$ be a gambler and $\zeta: \Nats \rightarrow (0,1]$ be s.t.

\begin{equation}
b_\zeta := \sum_{b=0}^\infty \zeta\left(b\right) b < \infty
\end{equation}

Define the gambler $\Bd_\zeta G$ by

\begin{equation}
\Bd_\zeta G := \sum_{b = 1}^\infty \zeta\left(b\right) \Bd_b G
\end{equation}

Then

\begin{equation}
\label{eqn:prp_furgal_gambler__svm_b_zeta}
\SVM \Bd_\zeta G_n \geq -b_\zeta
\end{equation}

Moreover, if $F$ is a forecaster that dominates $\Bd_\zeta G$ then it also dominates $G$.

\end{proposition}

\begin{proof}

Equation~\ref{eqn:prp_furgal_gambler__svm_b_zeta} follows from Proposition~\ref{prp:b_stays_in_budget}. Now, consider any $F$ that dominates $\Bd_\zeta G$. For any $x \in \OO$, \ref{eqn:prp_furgal_gambler__svm_b_zeta} and Definition~\ref{def:dominance} imply

\[\sup_{n \in \Nats} {\SVX \Bd_\zeta G^F_{n}\left(x_{:n-1}\right)} < +\infty\]

Consider $x \in \OO$ s.t. condition~\ref{eqn:def_dominance__loss} holds. Denote

\[b_0: = \left( -\inf_{n \in \Nats} {\SVM G^F_{n}\left(x_{:n-1}\right)}\right)_+\]

By Proposition~\ref{prp:b_no_effect}, for any $b > b_0$ we have $\Bd_b G^F = G^F$. We get

\[\sup_{n \in \Nats} \max_{x' \in x_{:n-1}\OO} \left(\sum_{b \leq b_0} \zeta\left(b\right) {\SV \Bd_b G^F_{n}\left(x_{:n-1},x'\right)} + \sum_{b > b_0} \zeta\left(b\right) \SV G^F_{n}\left(x_{:n-1},x'\right)\right) < +\infty\]

Applying Proposition~\ref{prp:b_stays_in_budget} to the first term

\[\sup_{n \in \Nats} \max_{x' \in x_{:n-1}\OO} \left(-\sum_{b \leq b_0} \zeta\left(b\right) b + \sum_{b > b_0} \zeta\left(b\right) \SV G^F_{n}\left(x_{:n-1},x'\right)\right) < +\infty\]

\[\sup_{n \in \Nats} \SVX G^F_{n}\left(x_{:n-1}\right) < +\infty\]
\end{proof}

The last ingredient we need to prove Theorem~\ref{thm:exist_dominant} is the step from one gambler to a countable set of gamblers.

\begin{proposition}
\label{prp:combining_gamblers}

Let $\Sq{G^k}{k}$ be a family of gamblers, $\xi: \Nats \rightarrow (0,1]$ and $b_\xi > 0$ s.t.

\begin{equation}
\label{eqn:prp_combining_gamblers__budget}
\sum_{k = 0}^\infty \xi\left(k\right) \min\left(\SVM G^k_n,0\right) \geq -b_\xi
\end{equation}

Define the gambler $G^\xi$ by

\begin{equation}
G^\xi_n := \sum_{k = 0}^n \xi\left(k\right) G^k_n
\end{equation}

Consider a forecaster $F$ that dominates $G^\xi$. Then, $F$ dominates $G^k$ for all $k \in \Nats$.

\end{proposition}

\begin{proof}

Fix $k \in \Nats$ and $x \in \OO$. Equation~\ref{eqn:prp_combining_gamblers__budget} implies

\[\inf_{n \in \Nats} \SVM G^{\xi F}_{n}\left(x_{:n-1},x'\right) > -\infty\]

Therefore we have

\[\sup_{n \in \Nats} \SVX G^{\xi F}_{n}\left(x_{:n-1}\right) < +\infty\]

\[\sup_{n \geq k} \max_{x' \in x_{:n-1}\OO} \left(\xi\left(k\right) \SV G^{k F}_{n}\left(x_{:n-1},x'\right) + \sum_{\substack{j \ne k\\ j \leq n}} \xi\left(j\right) \SV G^{j F}_{n}\left(x_{:n-1},x'\right)\right) < +\infty\]

\[\sup_{n \geq k} \max_{x' \in x_{:n-1}\OO} \left(\xi\left(k\right) \SV G^{k F}_{n}\left(x_{:n-1},x'\right) - b_\xi\right) < +\infty\]

\[\sup_{n \geq k} \SVX G^{k F}_{n}\left(x_{:n-1}\right) < +\infty\]

Since there are only finitely many values of $n$ less than $k$, it follows that

\[\sup_{n \in \Nats} \SVX G^{k F}_{n}\left(x_{:n-1}\right) < +\infty\]
\end{proof}

\begin{proof}[Proof of Theorem~\ref{thm:exist_dominant}]

Define $\zeta\left(b\right) := \max\left(b,1\right)^{-3}$. By Proposition~\ref{prp:frugal_gambler}, we have

\[\SVM \Bd_{\zeta} G^k_n \geq -\frac{\pi^2}{6}\]

Define $\xi\left(k\right):=\left(k+1\right)^{-2}$.

We have

\[\sum_{k=0}^\infty \xi\left(k\right) \min\left(\SVM \Bd_{\zeta} G^k_n, 0\right) \geq -\frac{\pi^4}{36} \]

By Corollary~\ref{crl:dominate_one}, there is a forecaster $F^G$ that dominates $\Bd_\zeta G^\xi$. By Proposition~\ref{prp:combining_gamblers}, $F^G$ dominates $\Bd_\zeta G^k$ for all $k \in \Nats$. By Proposition~\ref{prp:frugal_gambler}, $F^G$ dominates $G^k$ for all $k \in \Nats$.
\end{proof}

\section{Appendix: Prudent Gambling Strategies}
\label{sec:prudent}

In this section we construct a tool for proving asymptotic convergence results about dominant forecasters. We start by introducing a notion of \enquote{prudent gambling,} which requires that bets are only made when the expected payoff is a certain fraction of the risk. First, we need some notation.

Given $X,Y$ Polish spaces, $\mu \in \PM(X)$, $\pi: X \rightarrow Y$ Borel measurable and $g: X \rightarrow \Reals$ Borel measurable, we will use the notation 

\begin{equation}
\E_{\mu}\left[g \mid \pi^{-1}\left(y\right)\right]:=\E_{\left(\mu \mid \pi\right)\left(y\right)}\left[g\right]
\end{equation}

We remind that $\PO_n$ is the projection mapping from $\OO$ to $\Ob^n$ and $y\OO:= \left(\PO_{n}\right)^{-1}\left(y\right)$.

\begin{definition}

A \emph{gambling strategy} is a uniformly bounded family of measurable mappings

\[\Sqn{S_n: \Ob^n \rightarrow \GMO}\]

\end{definition}

As opposed to a gambler, a gambling strategy doesn't depend on the forecast history.

\begin{definition}
\label{def:prudent}

Given a gambling strategy $S$ and $\mu^* \in \PMO$, $S$ is said to be \emph{$\mu^*$-prudent} when there is $\alpha > 0$ s.t. for any $n \in \Nats$, $\PO_{n*} \mu^*$-almost any $y \in \Ob^n$ and any $\mu \in \PMO$

\begin{equation}
\E_{\mu^*}\left[S_n\left(y; \mu\right) \mid y\OO\right] - \E_\mu\left[S_n\left(y; \mu\right)\right] \geq \alpha \left(\max_{y\OO}{S_n\left(y; \mu\right)} - \min_{y\OO}{S_n\left(y; \mu\right)}\right)
\end{equation}

\end{definition}

Here, we can think of $\mu$ as the forecast and $\mu^*$ as the true environment. 

The \enquote{correct} way for a gambler to use a prudent gambling strategy is not employing it all the time: in order to avoid running out of budget, a gambler shouldn’t place too many bets simultaneously. To address this, we define a gambler that \enquote{plays} (employs the given strategy) only when all previous bets are close to being settled.

\begin{definition}

Consider any gambling strategy $S$. We define the gambler $\PG{S}$ recursively as follows

\begin{equation}
\PG{S}_n\left(y,\BM\right):=\begin{cases} S_n\left(y\right) \text{ if } {\SVX{\PG{S}}_n\left(y_{:n-1},\BM\right)}-{\SVM{\PG{S}}_n\left(y_{:n-1},\BM\right)} \leq 1 \\ 0 \text{ otherwise} \end{cases}
\end{equation}

\end{definition}

\begin{theorem}
\label{thm:prudent}

Consider $\mu^* \in \PMO$, $S$ a $\mu^*$-prudent gambling strategy and a forecaster $F$. Assume $F$ dominates $\PG{S}$. Then, for $\mu^*$-almost any $x \in \OO$

\begin{equation}
\label{eqn:thm_prudent}
\sum_{n=0}^\infty \left(\E_{\mu^*}\left[S_n\left(x_{:n};F_n\left(x_{:n}\right)\right) \mid x_{:n}\OO\right]-\E_{F_n\left(x_{:n}\right)}\left[S_n\left(x_{:n};F_n\left(x_{:n}\right)\right)\right]\right) < +\infty
\end{equation}

\end{theorem}

In particular, the terms of the series on the left hand-side converge to 0 (they are non-negative since $S$ is $\mu^*$-prudent).

The rest of the section is dedicated to proving Theorem~\ref{thm:prudent}.

We start by showing, in a more abstract setting (not specific to forecasting), that for any sequences of bets that is \enquote{prudent} in the sense that the expected payoff is bounded below by a fixed fraction of the risk (= bet magnitude), the probability to lose an infinite amount vanishes and, moreover, it holds with probability 1 that if the total bet magnitude is unbounded then there is infinite gain.

\begin{samepage}
\begin{lemma}
\label{lmm:prudent}

Consider $\Omega$ a probability space, $\Sqn{\F_n \subseteq 2^\Omega}$ a filtration of $\Omega$, $\alpha,M > 0$, $\Sqn{X_n: \Omega \rightarrow \Reals}$ and $\Sqn{Y_n: \Omega \rightarrow \left[0,M\right]}$ stochastic processes adapted to $\F$ (i.e. $X_n$ and $Y_n$ are $\F_n$-measurable). Assume the following conditions\footnote{Equalities and inequalities between random variables are implicitly understood as holding almost surely.}:

\begin{enumerate}[i.]

\item $\E\left[\A{X_0}\right] < +\infty$
\item $\forall n \in \Nats, m \geq n: \A{X_n - X_m} \leq \sum_{l=n}^{m-1} Y_l + M$
\item $\E\left[X_{n+1} \mid \F_n\right] \geq X_n + \alpha Y_n$

\end{enumerate}

Then, for almost all $\omega \in \Omega$

\begin{enumerate}[a.]

\item\label{itm:lmm_prudent__inf} $\inf_n X_n\left(\omega\right) > -\infty$
\item\label{itm:lmm_prudent__sup} If $\sup_n X_n\left(\omega\right) < +\infty$ then $\sum_n Y_n\left(\omega\right) < +\infty$

\end{enumerate}

\end{lemma}
\end{samepage}

\begin{proof}

Without loss of generality, we can assume ${M = 1}$ (otherwise we can renormalize ${X}$ and ${Y}$ by a factor of ${M^{-1}}$). Define ${\Sqn{X^0_n: \Omega \rightarrow \Reals}}$ by

$$X^0_n := X_n - \alpha \sum_{m=0}^{n-1} Y_m$$

It is easy to see that ${X^0}$ is a submartingale.

We define ${\Sqn{N_n:\Omega \rightarrow \Nats \sqcup \{\infty\}}}$ recursively as follows

\begin{align*}
N_0\left(\omega\right) &:= 0 \\ 
N_{k+1}\left(\omega\right) &:= \begin{cases}\inf \{n \in \Nats \mid \sum_{m=N_k\left(\omega\right)}^{n-1} Y_m\left(\omega\right) \geq 1\} \text{ if } N_k\left(\omega\right) < \infty\\\infty \text{ if } N_k\left(\omega\right) = \infty\end{cases}
\end{align*}

For each $k \in \Nats$, $N_k$ is a stopping time w.r.t. ${\F}$. As easy to see, for each $k \in \Nats$, ${\Sqn{X^0_{\min\left(n,N_k\right)}}}$ is a uniformly integrable submartingale. Using the fact that ${N_{k} \leq N_{k+1}}$ to apply Theorem~\ref{thm:optional_stopping} (see Appendix~\ref{sec:theorems}), we get

$$\E\left[X^0_{N_{k+1}} \mid \F_{N_k}\right] \geq X^0_{N_{k}}$$

Clearly, ${\Sq{X^0_{N_k}}{k}}$ is adapted to ${\Sq{\F_{N_k}}{k}}$. Doob's second martingale convergence theorem implies that ${\E\left[\A{X^0_{N_k}}\right] < \infty}$ (${X^0_{N_k}}$ is the limit of the uniformly integrable submartingale ${\Sqn{X^0_{\min\left(n,N_k\right)}}}$). We conclude that ${\Sq{X^0_{N_k}}{k}}$ is a submartingale.

It is easy to see that ${\A{X^0_{N_{k+1}}-X^0_{N_k}}} \leq 3$. Applying the Azuma-Hoeffding inequality, we conclude that for any positive integer ${k}$:

$$\Pr\left[X^0_{N_k} - X_0 < -k^{\frac{3}{4}}\right] \leq \exp\left(-\frac{k^{\frac{3}{2}}}{2 \cdot 3^2k}\right)=\exp\left(-\frac{k^{\frac{1}{2}}}{18}\right)$$

It follows that

$$\sum_{k=1}^\infty \Pr\left[X^0_{N_k} - X_0 < -k^{\frac{3}{4}}\right] < \infty$$

$$\Pr\left[\exists k \in \Nats \forall l > k: X^0_{N_l} - X_0 \geq -l^{\frac{3}{4}}\right] = 1$$

$$\Pr\left[\exists k \in \Nats \forall l > k: X_{N_l} - X_0 \geq \alpha \sum_{n=0}^{N_l - 1} Y_n - l^{\frac{3}{4}}\right] = 1$$

%$$\Pr\left[\exists m \in \Nats \forall k \in \Nats: X^0_{N_k} \leq m - \alpha k\right] = 0$$

It is easy to see that if $\omega \in \Omega$ is s.t. $\sum_{n=0}^\infty Y_n\left(\omega\right) = \infty$ then $\sum_{n=0}^{N_l\left(\omega\right) - 1} Y_n\left(\omega\right) \geq l$. We get

$$\Pr\left[\sum_{n=0}^\infty  Y_n = \infty \implies \exists k \in \Nats \forall l > k: X_{N_l} - X_0 \geq \alpha l - l^{\frac{3}{4}}\right]=1$$

To complete the proof, observe that if ${\omega \in \Omega}$ is s.t. ${\inf_n X_n\left(\omega\right) = -\infty}$ then ${\sum_{n=0}^\infty  Y_n\left(\omega\right) = \infty}$ (since ${\A{X_n\left(\omega\right) - X_0\left(\omega\right)} \leq \sum_{m=0}^{n-1} Y_n\left(\omega\right) + 1}$).
\end{proof}

\begin{samepage}
\begin{corollary}
\label{crl:prudent}

Consider $\Omega$ a probability space, $\Sqn{\F_n \subseteq 2^\Omega}$ a filtration of $\Omega$, $\alpha,M > 0$, $\Sqn{X'_n: \Omega \rightarrow \Reals}$ and $\Sqn{Y_n: \Omega \rightarrow \left[0,M\right]}$ stochastic processes adapted to $\F$ and $\Sqn{X_n: \Omega \rightarrow \Reals}$ an arbitrary stochastic process. Assume the following conditions:

\begin{enumerate}[i.]

\item $\A{X_n - X'_n} \leq M$
\item $\E\left[\A{X_0}\right] < +\infty$
\item $\A{X_{n+1} - X_n} \leq Y_n$
\item $\E\left[X_{n+1} - X_n \mid \F_n\right] \geq \alpha Y_n$

\end{enumerate}

Then, for almost all $\omega \in \Omega$

\begin{enumerate}[a.]

\item\label{itm:crl_prudent__inf} $\inf_n X_n\left(\omega\right) > -\infty$
\item\label{itm:crl_prudent__sup} If $\sup_n X_n\left(\omega\right) < +\infty$ then $\sum_n Y_n\left(\omega\right) < +\infty$

\end{enumerate}

\end{corollary}
\end{samepage}

\begin{proof}

Define ${X''_n:= \E\left[X_n \mid \F_n\right]}$. We have

$$\E\left[\A{X''_0}\right] = \E\left[\A{\E\left[X_0 \mid \F_0\right]}\right] \leq \E\left[\E\left[\A{X_0} \mid \F_0\right]\right] = \E\left[\A{X_0}\right] < \infty$$

$$\A{X''_n - X'_n} = \A{\E\left[X_n \mid \F_n\right] - X'_n} = \A{\E\left[X_n - X'_n \mid \F_n\right]} \leq M$$

$$\A{X''_n - X_n} \leq \A{X''_n - X'_n} + \A{X'_n - X_n} \leq 2M$$

$$\forall m \geq n: \A{X''_{m}-X''_n} \leq \A{X_{m}-X_n} + 4M \leq \sum_{l=n}^{n-1} Y_l + 4M$$

$$\E\left[X''_{n+1}-X''_n \mid \F_n\right] = \E\left[\E\left[X_{n+1} \mid \F_{n+1}\right]-\E\left[X_n \mid \F_n\right] \mid \F_n\right] = \E\left[X_{n+1}-X_n \mid \F_n\right] \geq \alpha Y_n$$

Applying Lemma~\ref{lmm:prudent} to $X''$ and using ${\A{X''_n - X_n} \leq 2M}$, we get the desired result.
\end{proof}

We will also need the following property of the operator $\PG$, reflecting that at any given moment, the uncertainty of the gambler $\PG{S}$ about the total worth of its bets is bounded.

\begin{samepage}
\begin{proposition}
\label{prp:pg_bounded_uncertainty}

Given any gambling strategy $S$, we have

\begin{equation}
\SV \PG{S}_n - \SVM \PG{S}_n \leq 2 \max_{m < n} \sup_{y \in \Ob^m} \N{S_m\left(y\right)} + 1
\end{equation}

In particular, the left hand side is uniformly bounded.

\end{proposition}
\end{samepage}

\begin{proof}

We prove by induction. For $n=0$ the claim is obvious. Consider any ${n \in \Nats}$, $y \in \Ob^{n}$ and $\BM \in \PMO^{n+1}$. First, assume that

$$\SVX \PG{S}_n\left(y_{:n-1},\BM_{:n}\right) - \SVM \PG{S}_n\left(y_{:n-1},\BM_{:n}\right) > 1$$

Then, by definition of ${\PG{S}}$, ${\PG{S}_n}\left(y,\BM_{:n}\right) \equiv 0$ and therefore

$$\SV \PG{S}_{n+1}\left(y,\BM\right)=\SV \PG{S}_n\left(y_{:n-1},\BM_{:n}\right)$$

$$\SVM \PG{S}_{n+1}\left(y,\BM\right) \geq \SVM \PG{S}_n\left(y_{:n-1},\BM_{:n}\right)$$

$$\SV \PG{S}_{n+1}\left(y,\BM\right) - \SVM \PG{S}_{n+1}\left(y,\BM\right) \leq \SV \PG{S}_n\left(y_{:n-1},\BM_{:n}\right) - \SVM \PG{S}_n\left(y_{:n-1},\BM_{:n}\right) $$

By the induction hypothesis, we get

$$\SV \PG{S}_{n+1}\left(y,\BM\right) - \SVM \PG{S}_{n+1}\left(y,\BM\right) \leq  2 \max_{m < n} \sup_{y \in \Ob^m} \N{S_m\left(y\right)}  + 1 \leq \max_{m \leq n} \sup_{y \in \Ob^m} \N{S_m\left(y\right)}  + 1$$

Now, assume that 

$$\SVX \PG{S}_n\left(y_{:n-1},\BM_{:n}\right) - \SVM \PG{S}_n\left(y_{:n-1},\BM_{:n}\right) \leq 1$$

Then, ${\PG{S}_n}\left(y,\BM_{:n}\right) = {S_n}\left(y\right)$ and therefore

$$\SV \PG{S}_{n+1}\left(y,\BM\right)=\SV \PG{S}_n\left(y_{:n-1},\BM_{:n}\right) + \left(\V{S}_n\left(y\right)\right)\left(\BM_n\right)$$

We get

$$\SV \PG{S}_{n+1}\left(y,\BM\right) - \SVM \PG{S}_{n+1}\left(y,\BM\right) \leq \SVX \PG{S}_n\left(y_{:n-1},\BM_{:n}\right) - \SVM \PG{S}_n\left(y_{:n-1},\BM_{:n}\right) + 2 \N{S_n\left(y\right)}$$

By the assumption, the first two terms total to ${\leq 1}$, yielding the desired result.
\end{proof}

As a final ingredient for the proof of Theorem~\ref{thm:prudent}, we will need another property of $\PG$, namely that when the total magnitude of all bets made is finite, the gambler eventually starts \enquote{playing} in every round.

\begin{samepage}
\begin{proposition}
\label{prp:when_bounded_play}

Let $S$ be a gambling strategy, $F$ a forecaster and $x \in \OO$. Assume that

\begin{equation}
\sum_{n=0}^\infty {\left(\max_{x_{:n}\OO} \overline{\PG{S}}^F_n\left(x_{:n}\right)-\min_{x_{:n}\OO} \overline{\PG{S}}^F_n\left(x_{:n}\right)\right)} < +\infty
\end{equation}

Then, for any $n \gg 0$, $\PG{S}^F_n\left(x_{:n}\right)=S_n\left(x_{:n}\right)$.

\end{proposition}
\end{samepage}

\begin{proof}

Choose $n_0 \in \Nats$ s.t. 

$$\sum_{n=n_0}^\infty {\left(\max_{x_{:n}\OO} \overline{\PG{S}}^F_n\left(x_{:n}\right)-\min_{x_{:n}\OO} \overline{\PG{S}}^F_n\left(x_{:n}\right)\right)} < \frac{1}{2}$$

Since $\overline{\PG{S}}^F_n\left(x_{:n}\right), \overline{\V{\PG{S}}}^F_n\left(x_{:n}\right) \in \CO$ differ by a constant function, we have

$$\sum_{n=n_0}^\infty {\left(\max_{x_{:n}\OO} \overline{\V{\PG{S}}}^F_n\left(x_{:n}\right)-\min_{x_{:n}\OO} \overline{\V{\PG{S}}}^F_n\left(x_{:n}\right)\right)} < \frac{1}{2}$$

In particular, for any $n \geq n_0$

$$\sum_{m=n_0}^{n-1} {\left(\max_{x_{:n}\OO} \overline{\V{\PG{S}}}^F_m\left(x_{:m}\right)-\min_{x_{:n}\OO} \overline{\V{\PG{S}}}^F_m\left(x_{:m}\right)\right)} < \frac{1}{2}$$

On the other hand, it is easy to see that there is $n_1 \geq n_0$ s.t. for any $n \geq n_1$

$$\max_{x_{:n}\OO} \SV{\PG{S}}_{n_0}^F\left(x_{:n_0 - 1}\right) - \min_{x_{:n}\OO} \SV{\PG{S}}_{n_0}^F\left(x_{:n_0 - 1}\right) < \frac{1}{2}$$

Taking the sum of the last two inequalities, we conclude that for any $n \geq n_1$

$$\SVX{\PG{S}}_{n}^F\left(x_{:n-1}\right) - \SVM{\PG{S}}_{n}^F\left(x_{:n-1}\right) < 1$$

Using the definition of $\PG{S}$, we get the desired result.
\end{proof}

\begin{proof}[Proof of Theorem~\ref{thm:prudent}]

We regard ${\OO}$ as a probability space using the ${\sigma}$-algebra $\F$ of universally measurable sets and the probability measure ${\mu^*}$. For any ${n \in \Nats}$, we define ${\F_n \subseteq \F}$ and $X_n,X'_n,Y_n: \OO \rightarrow \Reals$ by 

$$\F_n := \{A\Ob^\omega \mid A \subseteq \Ob^n \text{ universally measurable}\}$$

$$X_n\left(x\right):= \SV \PG{S}^F_{n}\left(x_{:n-1},x\right)$$

$$X'_n\left(x\right):= \SVM \PG{S}^F_{n}\left(x_{:n-1}\right)$$

$$Y_n\left(x\right):=\max_{x_{:n}\OO} \overline{\PG{S}}^F_{n}\left(x_{:n}\right) - \min_{x_{:n}\OO} \overline{\PG{S}}^F_{n}\left(x_{:n}\right)$$

Clearly, ${\F}$ is a filtration of ${\OO}$, ${X,X',Y}$ are stochastic processes and ${X',Y}$ are adapted to ${\F}$. ${S}$ is uniformly bounded, therefore ${\PG{S}}$  is uniformly bounded and so is ${Y}$. Obviously, ${Y}$ is also non-negative.

By Proposition~\ref{prp:pg_bounded_uncertainty}, ${\A{X_n-X'_n}}$ are uniformly bounded. ${X_0}$ vanishes and in particular ${\E\left[\A{X_0}\right] < \infty}$. We have

$$\A{X_{n+1}\left(x\right)-X_n\left(x\right)} = \A{\overline{\V{\PG{S}}}_{n}^F\left(x_{:n},x\right)} \leq Y_n\left(x\right)$$

Moreover:

$$\E\left[X_{n+1} - X_n \mid \F_n\right]\left(x\right) = \E_{x' \sim \mu^*}\left[\overline{\V{\PG{S}}}_{n}^F\left(x_{:n},x'\right) \mid x' \in x_{:n}\OO\right]$$

(As before, the notation on the right hand side signifies the use of a regular conditional probability)

$$\E\left[X_{n+1} - X_n \mid \F_n\right]\left(x\right) = \E_{x' \sim \mu^*}\left[\overline{\PG{S}}_{n}^F\left(x_{:n},x'\right) \mid x' \in x_{:n}\OO\right] - \E_{x' \sim F_n\left(x_{:n}\right)}\left[\overline{\PG{S}}_{n}^F\left(x_{:n},x'\right)\right]$$

Let $\alpha > 0$ be as in Definition~\ref{def:prudent}. By definition of $\PG{S}$, $\overline{\PG{S}}_{n}^F\left(x_{:n}\right)$ is equal to either $S_n\left(x_{:n};F_n\left(x_{:n}\right)\right)$ or 0. In either case, we get that for $\mu^*$-almost any $x \in \OO$

$$\E\left[X_{n+1} - X_n \mid \F_n\right]\left(x\right) \geq \alpha \left(\max_{x_{:n}\OO} \overline{\PG{S}}_{n}^F\left(x_{:n}\right) - \min_{x_{:n}\OO} \overline{\PG{S}}_{n}^F\left(x_{:n}\right)\right) = \alpha Y_n\left(x\right)$$

By Corollary~\ref*{crl:prudent}\ref{itm:crl_prudent__inf}, for $\mu^*$-almost any $x \in \OO$

\[\inf_{n \in \Nats} \SVM \PG{S}^F_{n}\left(x_{:n-1}\right) > -\infty\]

Since $F$ dominates $\PG{S}$, it follows that

\[\sup_{n \in \Nats} \SVX \PG{S}^F_{n}\left(x_{:n-1}\right) < +\infty\] 

By Corollary~\ref*{crl:prudent}\ref{itm:crl_prudent__sup}, $\sum_n Y_n\left(x\right) < \infty$. By Proposition~\ref{prp:when_bounded_play}, it follows that for any $n \gg 0$, $\PG{S}^F_n\left(x_{:n}\right) = S_n\left(x_{:n}\right)$. We get, for $n \gg 0$

\begin{align*}
&\E_{\mu^*}\left[S_n\left(x_{:n};F_n\left(x_{:n}\right)\right) \mid x_{:n}\OO\right]-\E_{F_n\left(x_{:n}\right)}\left[S_n\left(x_{:n};F_n\left(x_{:n}\right)\right)\right] = \\ 
&\E_{\mu^*}\left[\overline{\PG{S}}^F_n\left(x_{:n}\right) \mid x_{:n}\OO\right]-\E_{F_n\left(x_{:n}\right)}\left[\overline{\PG{S}}^F_n\left(x_{:n}\right)\right] \leq Y_n\left(x\right)
\end{align*}

$$\sum_{n=0}^\infty \left(\E_{\mu^*}\left[S_n\left(x_{:n};F_n\left(x_{:n}\right)\right) \mid x_{:n}\OO\right]-\E_{F_n\left(x_{:n}\right)}\left[S_n\left(x_{:n};F_n\left(x_{:n}\right)\right)\right]\right) < +\infty$$
\end{proof}

\section{Appendix: Proof of Main Theorem}
\label{sec:construction}

In order to prove Theorem~\ref{thm:main}, we will need, given an incomplete model $M \subseteq \PMO$ and $\epsilon > 0$, a gambling strategy $S^{M\epsilon}$ s.t. $S^{M\epsilon}$ is $\mu$-prudent for any $\mu \in M$ and s.t. Theorem~\ref{thm:prudent} implies that any $F$ that dominates $S^{M\epsilon}$ satisfies equation~\ref{eqn:thm_main} \enquote{within $\epsilon$.} This motivates the following definition and lemma. We remind that, given $x \in \Reals$, $x_+:=\max(x,0)$.

\begin{samepage}
\begin{definition}

Consider $X$ a compact Polish metric space, $M \subseteq \PM\left(X\right)$ convex and $r_0 > 0$. $\beta \in \Gm\left(X\right)$ is said to be \emph{$\left(M,r_0\right)$-savvy} when for any $\mu \in \PM\left(X\right)$, denoting $r_\mu:=\DKR\left(\mu,M\right)$:

\begin{enumerate}[i.]

\item $\N{\beta\left(\mu\right)} \leq \left(r_\mu - r_0\right)_+$
\item $\forall \nu \in M: \E_\nu\left[\beta\left(\mu\right)\right] - \E_\mu\left[\beta\left(\mu\right)\right] \geq \frac{1}{2}\left(r_\mu - r_0\right) r_\mu$

\end{enumerate}

\end{definition}
\end{samepage}

\begin{samepage}
\begin{lemma}
\label{lmm:savvy}

Consider $X$ a compact Polish metric space, $Y$ a compact Polish space and $M: Y \rightarrow 2^{\PM\left(X\right)}$. Denote

\begin{equation*}
\Gr{M}:=\{\left(y, \mu\right) \in Y \times \PM\left(X\right) \mid \mu \in M\left(y\right)\}
\end{equation*}

Assume that $\Gr{M}$ is closed and that $M\left(y\right)$ is convex for any $y \in Y$. Regard $Y$ as a measurable space using the $\sigma$-algebra of universally measurable sets and regard $\Gm\left(X\right)$ as a measurable space using the $\sigma$-algebra of Borel sets. Then, for any $\epsilon > 0$, there exists $S^\epsilon: Y \rightarrow \Gm\left(X\right)$ measurable s.t. for all $y \in Y$, $S^\epsilon\left(y\right)$ is $\left(M\left(y\right),\epsilon\right)$-savvy.

\end{lemma}
\end{samepage}

In order to prove Lemma~\ref{lmm:savvy}, we will need a few other technical lemmas.

\begin{samepage}
\begin{lemma}
\label{lmm:double_dual}

Consider $X$ a compact Polish space, $\rho$ a metrization of $X$, $\varphi \in \Lp\left(X,\rho\right)''$ and $\epsilon > 0$. Then, there exists $f \in \Lp\left(X,\rho\right)$ s.t.

\begin{enumerate}[i.]

\item $\N{f}_\rho \leq \N{\varphi}$
\item $\forall \mu \in \PM\left(X\right): \A{\E_\mu\left[f\right] - \varphi\left(\mu\right)} < \epsilon$

\end{enumerate}

\end{lemma}
\end{samepage}

\begin{proof}

Without loss of generality, assume that $\N{\varphi}=1$ (if $\N{\varphi}=0$ the theorem is trivial, otherwise we can rescale everything by a scalar). Using the compactness of $\PM\left(X\right)$, choose a finite set $A \subseteq \PM\left(X\right)$ s.t. 

$$\max_{\mu \in \PM\left(X\right)} \min_{\nu \in A}\, \DKR\left(\mu,\nu\right) < \frac{\epsilon}{4}$$

Using the Goldstine theorem, choose $f \in \Lp\left(X\right)$ s.t. $\N{f}_\rho \leq 1$ and for any $\nu \in A$ 

$$\A{\E_\nu\left[f\right] - \varphi\left(\nu\right)} < \frac{\epsilon}{2}$$

Consider any $\mu \in \PM\left(X\right)$. Choose $\nu \in A$ s.t. $\DKR\left(\mu,\nu\right) < \frac{\epsilon}{4}$. We get

$$\A{\E_\mu\left[f\right] - \varphi\left(\mu\right)} \leq \A{\E_\mu\left[f\right] - \E_\nu\left[f\right]} + \A{\E_\nu\left[f\right] - \varphi\left(\nu\right)} + \A{\varphi\left(\nu\right) - \varphi\left(\mu\right)} < \frac{\epsilon}{4} + \frac{\epsilon}{2} + \frac{\epsilon}{4} = \epsilon$$
\end{proof}

\begin{samepage}
\begin{lemma}
\label{lmm:separation}

Consider $X$ a compact Polish space, $\rho$ a metrization of $X$, $M \subseteq \PM\left(X\right)$ convex non-empty, $r_0 > 0$ and $\mu_0 \in \PM\left(X\right)$ s.t. $r:=\DKR\left(\mu_0,M\right) > r_0$. Then, there exists $f \in \Lp\left(X,\rho\right)$ s.t.

\begin{enumerate}[i.]

\item $\N{f}_\rho < r - r_0$
\item $\inf_{\nu \in M} \E_\nu\left[f\right] - \E_{\mu_0}\left[f\right] > \frac{1}{2}\left(r-r_0\right)r$

\end{enumerate}

\end{lemma}
\end{samepage}

\begin{proof}

Define $M^r \subseteq \Lp\left(X,\rho\right)'$ by

$$M^r:=\{\mu \in \Lp\left(X,\rho\right)' \mid \inf_{\nu \in M} \N{\mu - \nu} < r\}$$

By the Hahn-Banach separation theorem, there is $\varphi \in \Lp\left(X,\rho\right)''$ s.t. for all $\nu \in M^r$, $\varphi\left(\nu\right) > \varphi\left(\mu_0\right)$. Multiplying by a scalar, we can make sure that $\N{\varphi} = \frac{3}{4}\left(r - r_0\right)$. For any $\delta > 0$, we can choose $\zeta \in \Lp\left(X,\rho\right)'$ s.t. $\N{\zeta} < 1$ and $\varphi\left(\zeta\right) > \frac{3}{4}\left(r-r_0\right) - \delta$.  For any $\nu \in M$, $\nu - r \zeta \in M^r$ and therefore

$$\varphi\left(\nu - r \zeta\right) > \varphi\left(\mu_0\right)$$

$$\varphi\left(\nu\right) > \varphi\left(\mu_0\right) + r \varphi\left(\zeta\right) > \varphi\left(\mu_0\right) + r \left(\frac{3}{4}\left(r-r_0\right) - \delta\right)$$

Taking $\delta$ to 0 we get $\varphi\left(\nu\right) \geq \varphi\left(\mu_0\right) + \frac{3}{4} r \left(r - r_0\right)$. Applying Lemma~\ref{lmm:double_dual} for $\epsilon = \frac{1}{16} r \left(r - r_0\right)$, we get $f \in \Lp\left(X,\rho\right)$ s.t.

$$\N{f}_\rho \leq \frac{3}{4} \left(r - r_0\right) < r - r_0$$ 

$$\E_\nu\left[f\right] \geq \varphi\left(\nu\right) - \frac{1}{16} r \left(r - r_0\right) \geq \varphi\left(\mu_0\right) + \frac{11}{16} r \left(r - r_0\right) \geq \E_{\mu_0}\left[f\right] +  \frac{5}{8} r \left(r - r_0\right) > \E_{\mu_0}\left[f\right] +  \frac{1}{2} r \left(r - r_0\right)$$
\end{proof}

\begin{samepage}
\begin{lemma}
\label{lmm:savvy_outside}

Consider $X$ a compact Polish space, $\rho$ a metrization of $X$, $M \subseteq \PM\left(X\right)$ convex non-empty and $r_0 > 0$. Define $U \subseteq \PM\left(X\right)$ by

\begin{equation*}
U := \{\mu \in \PM\left(X\right) \mid \DKR\left(\mu,M\right) > r_0\}
\end{equation*}

Then, there exists $\beta': U \rightarrow \Lp\left(X,\rho\right)$ continuous s.t. for all $\mu \in U$, denoting $r_\mu:=\DKR\left(\mu,M\right)$:

\begin{enumerate}[i.]

\item $\N{\beta'\left(\mu\right)}_\rho < r_\mu - r_0$
\item $\inf_{\nu \in M} \E_\nu\left[\beta'\left(\mu\right)\right] - \E_\mu\left[\beta'\left(\mu\right)\right] > \frac{1}{2} \left(r_\mu - r_0\right) r_\mu$

\end{enumerate}

\end{lemma}
\end{samepage}

\begin{proof}

Define $B: U \rightarrow 2^{\Lp\left(X,\rho\right)}$ by

$$B\left(\mu\right):=\{f \in \Lp\left(X,\rho\right) \mid \N{f}_\rho < r_\mu - r_0,\, \inf_{\nu \in M} \E_\nu\left[f\right] - \E_\mu\left[f\right] > \frac{1}{2}\left(r_\mu - r_0\right) r_\mu\}$$

By Lemma~\ref{lmm:separation}, for any $\mu \in U$, $B\left(\mu\right) \ne \varnothing$. Clearly, $B\left(\mu\right)$ is convex. Fix $f \in \Lp\left(X,\rho\right)$ and consider $B^{-1}\left(f\right):=\{\mu \in U \mid f \in B\left(\mu\right)\}$. Consider any $\mu_0 \in B^{-1}\left(f\right)$, and take $\epsilon > 0$ s.t.

\begin{enumerate}[i.]

\item $\N{f}_\rho < r_{\mu_0}  - \epsilon - r_0$
\item $\inf_{\nu \in M} \E_\nu\left[f\right] - \E_{\mu_0}\left[f\right] -\epsilon > \frac{1}{2}\left(r_{\mu_0} + \epsilon  - r_0\right) \left(r_{\mu_0} + \epsilon\right)$

\end{enumerate}

Define $V \subseteq U$ by 

$$V := \{\mu \in U \mid \DKR\left(\mu,\mu_0\right) < \epsilon,\, \E_\mu\left[f\right] < \E_{\mu_0}\left[f\right] + \epsilon\}$$

Obviously $V$ is open, $\mu_0\in V$ and $V \subseteq B^{-1}\left(f\right)$. We got that $\mu_0$ has an open neighborhood inside $B^{-1}(f)$, and since we could choose any $\mu_0\in B^{-1}(f)$, we showed that $B^{-1}\left(f\right)$ is open. Applying Theorem~\ref{thm:selection} (see Appendix~\ref{sec:theorems}), we get the desired result.
\end{proof}

\begin{samepage}
\begin{corollary}
\label{crl:savvy}

Consider $X$ a compact Polish metric space, $M \subseteq \PM\left(X\right)$ convex and $r_0 > 0$. Then, there exists $\beta \in \Gm\left(X\right)$ which is $\left(M,r_0\right)$-savvy.

\end{corollary}
\end{samepage}

\begin{proof}

If $M$ is empty, the claim is trivial, so assume $M$ is non-empty. Let $U \subseteq \PM\left(X\right)$ be defined by

$$U:=\{\mu \in \PM\left(X\right) \mid \DKR\left(\mu, M\right) > r_0\}$$

Use Lemma~\ref{lmm:savvy_outside} to obtain $\beta': U \rightarrow \Lp\left(X\right)$. Define $\beta: \PM\left(X\right) \rightarrow \Lp\left(X\right)$ by

$$\beta\left(\mu\right):=\begin{cases}\beta'\left(\mu\right) \text{ if } \mu \in U\\0 \text { if } \mu \not\in U\end{cases}$$ 
Obviously $\beta$ is continuous in $U$. Consider any $\mu \not\in U$ and $\Sq{\mu_k \in \PM\left(X\right)}{k}$ s.t. $\lim_{k \rightarrow \infty} \mu_k = \mu$. We have 

$$\lim_{k \rightarrow \infty} \DKR\left(\mu_k,M\right) = \DKR\left(\mu, M\right) \leq r_0$$

Denote the metric on $X$ by $\rho$. 

$$\limsup_{k \rightarrow \infty}{\N{\beta\left(\mu_k\right)}_\rho} \leq \limsup_{k \rightarrow \infty}{\left(\DKR\left(\mu_k,M\right) - r_0\right)_+} = 0$$

Therefore, $\beta$ is continuous everywhere.
\end{proof}

\begin{proof}[Proof of Lemma~\ref{lmm:savvy}]

Define $Z_{1,2,3} \subseteq Y \times \Gm\left(X\right) \times \PM\left(X\right)^4$ by

$$Z_1:=\{\left(y,\beta,\mu,\nu,\xi,\zeta\right) \in Y \times \Gm\left(X\right) \times \PM\left(X\right)^4 \mid \zeta \in M\left(y\right)\}$$

$$Z_2:=\{\left(y,\beta,\mu,\nu,\xi,\zeta\right) \in Y \times \Gm\left(X\right) \times \PM\left(X\right)^4 \mid \N{\beta\left(\mu\right)} \leq \left(\DKR\left(\mu,\xi\right) - \epsilon\right)_+\}$$

$$Z_3:=\{\left(y,\beta,\mu,\nu,\xi,\zeta\right) \in Y \times \Gm\left(X\right) \times \PM\left(X\right)^4 \mid \E_\nu\left[\beta\left(\mu\right)\right] - \E_\mu\left[\beta\left(\mu\right)\right] \geq \frac{1}{2}\left(\DKR\left(\mu,\zeta\right) - \epsilon\right)_+ \DKR\left(\mu,\zeta\right)\}$$

Define $Z := Z_1 \cap Z_2 \cap Z_3$. It is easy to see that $Z_{1,2,3}$ are closed and therefore $Z$ also. Define $Z' \subseteq Y \times \Gm\left(X\right) \times \PM\left(X\right)^3$ by

$$Z':=\{\left(y,\beta,\mu,\nu,\xi\right) \in Y \times \Gm\left(X\right) \times \PM\left(X\right)^3 \mid \exists \zeta \in \PM\left(X\right): \left(y,\beta,\mu,\nu,\xi,\zeta\right) \in Z\}$$

$Z'$ is closed since it is the projection of $Z$ and $\PM\left(X\right)$ is compact. The $M\left(y\right)$ are compact, therefore $\DKR\left(\mu,M\left(y\right)\right)=\DKR\left(\mu,\zeta\right)$ for some $\zeta \in M\left(y\right)$, and hence $\left(y,\beta,\mu,\nu,\xi\right) \in Z'$ if and only if the following conditions hold:

\begin{enumerate}[i.]

\item $\N{\beta\left(\mu\right)} \leq \left(\DKR\left(\mu,\xi\right) - \epsilon\right)_+$
\item $\E_\nu\left[\beta\left(\mu\right)\right] - \E_\mu\left[\beta\left(\mu\right)\right] \geq \frac{1}{2}\left(\DKR\left(\mu,M\left(y\right)\right) - \epsilon\right)_+ \DKR\left(\mu,M\left(y\right)\right)$

\end{enumerate}

Define $W \subseteq Y \times \Gm\left(X\right) \times \PM\left(X\right)^3$ by

$$W:=\{\left(y,\beta,\mu,\nu,\xi\right) \in Y \times \Gm\left(X\right) \times \PM\left(X\right)^3 \mid \nu,\xi \in M\left(y\right),\, \left(y,\beta,\mu,\nu,\xi\right) \not\in Z'\}$$

$W$ is locally closed and in particular it is an $F_\sigma$ set. Define $W' \subseteq Y \times \Gm\left(X\right)$ by

$$W':=\{\left(y,\beta\right) \in Y \times \Gm\left(X\right) \mid \exists \mu,\nu,\xi \in \PM\left(X\right): \left(y,\beta,\mu,\nu,\xi\right) \in W\}$$

$W'$ is the projection of $W$ and $\PM\left(X\right)^3$ is compact, therefore $W'$ is $F_\sigma$. Let $A \subseteq Y \times \Gm\left(X\right)$ be the complement of $W'$. $A$ is $G_\delta$ and in particular Borel. As easy to see, $\left(y,\beta\right) \in A$ if and only if $\beta$ is $\left(M\left(y\right),\epsilon\right)$-savvy.

For any $y \in Y$, $A_y:=\{\beta \in \Gm\left(X\right) \mid \left(y,\beta\right) \in A\}$ is closed since

$$A_y = \bigcap_{\mu \in \PM\left(X\right)} \bigcap_{\nu,\xi \in M\left(y\right)} \{\beta \in \Gm\left(X\right) \mid \left(y,\beta,\mu,\nu,\xi\right) \in Z'\}$$

Moreover, $A_y$ is non-empty by Corollary~\ref{crl:savvy}.

Consider any $U \subseteq \Gm\left(X\right)$ open. Then, $A \cap \left(Y \times U\right)$ is Borel and therefore its image under the projection to $Y$ is analytic and in particular universally measurable. Applying the Kuratowski--Ryll-Nardzewski measurable selection theorem, we obtain a measurable selection of the multivalued mapping with graph $A$. This selection is the desired $S^\epsilon$.
\end{proof}

We now obtain a relation between the notions of savviness and prudence. We remind that given $X,Y$ Polish, $\pi: X \rightarrow Y$ Borel measurable and $\mu \in \PM\left(X\right)$, $\mu \mid \pi: Y \M\ X$ stands for any corresponding regular conditional probability.

\begin{samepage}
\begin{proposition}
\label{crl:savvy_is_prudent}

Consider $M \subseteq \PMO$, $\Sqn{M_n: \Ob^n \rightarrow 2^{\PMO}}$ and $\epsilon > 0$. Assume $M_n$ is a regular upper bound for $M \mid \PO_n$. Let $S$ be a gambling strategy s.t. for each $n \in \Nats$ and $y \in \Ob^n$, $S_n\left(y\right)$ is $\left(M_n(y),\epsilon\right)$-savvy (relatively to $\rho_n$). Then, $S$ is $\mu^*$-prudent for any $\mu^* \in M$.

\end{proposition}
\end{samepage}

\begin{proof}

Fix any $\mu^* \in M$. By Definition~\ref{def:update_incomplete}, for $\PO_{n*}\mu$-almost any $y\in\Ob^n$, ${\left(\mu^* \mid \PO_n\right)\left(y\right) \in M_n(y)}$. Since $S_n\left(y\right)$ is $\left(M_n(y),\epsilon\right)$-savvy, for any $\mu \in \PM\left(X\right)$, denoting $r_{\mu y}:=\DKR^n\left(\mu,\,M_n(y)\right)$:

$$\E_{\mu^*}\left[S_n\left(y;\mu\right) \mid y\OO\right] - \E_{\mu}\left[S_n\left(y;\mu\right)\right] \geq \frac{1}{2} \left(r_{\mu y} - \epsilon\right)_+ r_{\mu y} \geq \frac{1}{2} \N{S_n\left(y;\mu\right)} r_{\mu y}$$

When $r_{\mu y} \leq \epsilon$, $S_n\left(y;\mu\right) = 0$, hence for any $\mu \in \PM\left(X\right)$, $\N{S_n\left(y;\mu\right)} r_{\mu y} \geq \N{S_n\left(y;\mu\right)} \epsilon$. We get

$$\E_{\mu^*}\left[S_n\left(y;\mu\right) \mid y\OO\right] - \E_{\mu}\left[S_n\left(y;\mu\right)\right] \geq \frac{\epsilon}{2} \N{S_n\left(y;\mu\right)}$$
\end{proof}

\begin{samepage}
\begin{corollary}
\label{crl:vicinity_convergence}

Consider $M \subseteq \PMO$, $\Sqn{M_n : \Ob^n \rightarrow 2^{\PMO}}$ and $\epsilon > 0$. Assume $M_n$ is a regular upper bound for $M \mid \PO_n$. Let $S$ be a gambling strategy s.t. for each $n \in \Nats$ and $y \in \Ob^n$, $S_n\left(y\right)$ is $\left(M_n(y),\epsilon\right)$-savvy. Let $F$ be a forecaster that dominates $\PG{S}$. Then, for any $\mu \in M$ and $\mu$-almost any $x \in \OO$

\begin{equation}
\limsup_{n \rightarrow \infty} \DKR^n\left(F_n\left(x_{:n}\right),\, M_n\left(x_{:n}\right)\right) \leq \epsilon
\end{equation}

\end{corollary}
\end{samepage}

\begin{proof}

By Proposition~\ref{crl:savvy_is_prudent}, $S$ is $\mu$-prudent. By Theorem~\ref{thm:prudent}, for $\mu$-almost any $x \in X$

$$\lim_{n \rightarrow \infty} {\left(\E_{\mu}\left[S_n\left(x_{:n};F_n\left(x_{:n}\right)\right) \mid x_{:n}\OO\right]-\E_{F_n\left(x_{:n}\right)}\left[S_n\left(x_{:n};F_n\left(x_{:n}\right)\right)\right]\right)} = 0$$

On the other hand, $\left(\mu \mid \PO_n\right)\left(x_{:n}\right) \in M_n\left(x_{:n}\right)$ and hence, since $S_n\left(x_{:n}\right)$ is $\left(M_n\left(x_{:n}\right),\epsilon\right)$-savvy, denoting $r_n:=\DKR^n\left(F_n\left(x_{:n}\right), M_n\left(x_{:n}\right)\right)$: 

$$\E_{\mu}\left[S_n\left(x_{:n};F_n\left(x_{:n}\right)\right) \mid x_{:n}\OO\right]-\E_{F_n\left(x_{:n}\right)}\left[S_n\left(x_{:n};F_n\left(x_{:n}\right)\right)\right] \geq \frac{1}{2} \left(r_{n} - \epsilon\right) r_n$$

We get

$$\limsup_{n \rightarrow \infty} {\left(r_{n} - \epsilon\right) r_n} \leq 0$$

$$\limsup_{n \rightarrow \infty} {r_n} \leq \epsilon$$
\end{proof}

We are finally ready to complete the proof of the main theorem.

\begin{proof}[Proof of Theorem~\ref{thm:main}]

Consider any $M \in \MC$ and a positive integer $k$. By Lemma~\ref{lmm:savvy}, for any $n \in \Nats$, there exists ${S^{Mk}_{n}: \Ob^n \rightarrow \GMO}$ measurable s.t. for all $y \in \Ob^n$, $S^{Mk}_{n}\left(y\right)$ is $\left(M_n(y),\frac{1}{k}\right)$-savvy. We fix such a $S^{Mk}_{n}$ for each $M$ and $k$. It is easy to see that the $\DKR^n$-diameter of $\PMO$ is at most 2, therefore $\N{S^{Mk}_{n}\left(y\right)} \leq 2$ and $\Sqn{S^{Mk}_{n}}$ is a gambling strategy. By Theorem~\ref{thm:exist_dominant}, there is a forecaster $F^\MC$ that dominates $\PG{S}^{Mk}$ for all $M$ and $k$. By Corollary~\ref{crl:vicinity_convergence}, for any $M \in \MC$, $\mu \in M$, positive integer $k$ and $\mu$-almost any $x \in X$

$$\limsup_{n \rightarrow \infty} {\DKR^n\left(F^\MC_n\left(x_{:n}\right),M_n\left(x_{:n}\right)\right)} \leq \frac{1}{k}$$

It follows that

$$\lim_{n \rightarrow \infty} {\DKR^n\left(F^\MC_n\left(x_{:n}\right),M_n\left(x_{:n}\right)\right)} = 0$$
\end{proof}

\section{Appendix: Some Useful Theorems}
\label{sec:theorems}

The following variant of the Optional Stopping Theorem appears in \cite{Durrett_2010} as Theorem 5.4.7.

\begin{samepage}
\begin{theorem}
\label{thm:optional_stopping}

Let ${\Omega}$ be a probability space, ${\Sqn{\F_n \subseteq 2^\Omega}}$ a filtration of ${\Omega}$, ${\Sqn{X_n: \Omega \rightarrow \Reals}}$ a stochastic process adapted to ${\F}$ and ${M,N: \Omega \rightarrow \Nats \sqcup \{\infty\}}$ stopping times (w.r.t. ${\F}$) s.t. ${M \leq N}$. Assume that ${\Sqn{X_{\min\left(n,N\right)}}}$ is a uniformly integrable submartingale. Using Doob's martingale convergence theorem, we can define ${X_N : \Omega \rightarrow \Reals}$ by\footnote{The limit in equation~\ref{eqn:thm_optional_stopping__xn} is converges \emph{almost} surely, which is sufficient for our purpose.}

\begin{equation}
\label{eqn:thm_optional_stopping__xn}
X_N\left(x\right):=\lim_{n \rightarrow \infty} X_{\min\left(n,N\right)}\left(x\right)=\begin{cases}X_{N\left(x\right)}\left(x\right) \text{ if } N\left(x\right) < \infty\\\lim_{n \rightarrow \infty} X_n\left(x\right) \text{ if } N\left(x\right) = \infty\end{cases}
\end{equation}

We define ${X_M: \Omega \rightarrow \Reals}$ is an analogous way\footnote{This time we can't use Doob's martingale convergence theorem, but whenever ${M\left(x\right) = \infty}$ we also have ${N\left(x\right) = \infty}$, therefore the limit still almost surely converges.}. We also define the $\sigma$-algebra ${\F_M \subseteq 2^\Omega}$ by

\begin{equation}
\F_M:=\{A \subseteq X \text{ measurable} \mid \forall n \in \Nats: A \cap M^{-1}\left(n\right) \in \F_n\}
\end{equation}

Then:

\begin{equation}
\E\left[X_N \mid \F_M\right] \geq X_M
\end{equation}

\end{theorem}
\end{samepage}

The following variant of the Michael selection theorem appears in \cite{Yannelis_1983} as Theorem 3.1.

\begin{samepage}
\begin{theorem} [Yannelis and Prabhakar]
\label{thm:selection}

Consider $X$ a paracompact Hausdorff topological space and $Y$ a topological vector space. Suppose $B: X \rightarrow 2^Y$ is s.t.

\begin{enumerate}[i.]

\item For each $x \in X$, $B\left(x\right) \ne \varnothing$.
\item For each $x \in X$, $B\left(x\right)$ is convex.
\item For each $y \in Y$, $\{x \in X \mid y \in B\left(x\right)\}$ is open.

\end{enumerate}

Then, there exists $\beta: X \rightarrow Y$ continuous s.t. for all $x \in X$, $\beta\left(x\right) \in B\left(x\right)$.

\end{theorem}
\end{samepage}

\section{Appendix: Examples of Regular Upper Bounds}
\label{sec:examples}

\begin{proof}[Proof of Example~\ref{exm:update_incomplete_finite}]

In this case, condition~\ref{con:def__update_incomplete__clos} means that $N$ must be \emph{pointwise} closed and condition~\ref{con:def__update_incomplete__cond} means that for every $\mu \in M$ and $y \in Y$ s.t. $\mu\left(\pi^{-1}(y)\right) > 0$, $\mu \mid \pi^{-1}(y) \in N(y)$. Since all three conditions are closed w.r.t. arbitrary set intersections, there is a unique minimum (it is the intersection of all multivalued mappings satisfying the conditions). It remains to show that $N(y)$ defined by equation \ref{eqn:exm__update_incomplete_finite} is convex. Fix $y \in Y$ and denote $A$ the set appearing on the right hand side of equation \ref{eqn:exm__update_incomplete_finite} under the closure. Consider any $\mu_1,\mu_2 \in M$ s.t. $\mu_1\left(\pi^{-1}(y)\right),\mu_2\left(\pi^{-1}(y)\right) > 0$ and $p \in (0,1)$. We have 

\[p\mu_1+(1-p)\mu_2 \in M\]

\[\left(p\mu_1+(1-p)\mu_2\right) \mid \pi^{-1}(y) \in A\]

\[\frac{\mu_1\left(\pi^{-1}(y)\right)p\left(\mu_1\mid \pi^{-1}(y)\right)+\mu_2\left(\pi^{-1}(y)\right)(1-p)\left(\mu_2\mid \pi^{-1}(y)\right)}{\mu_1\left(\pi^{-1}(y)\right)p+\mu_2\left(\pi^{-1}(y)\right)(1-p)} \in A\]

It is easy to see that for any $q \in (0,1)$ there is $p \in (0,1)$ s.t.

\[q = \frac{\mu_1\left(\pi^{-1}(y)\right)p}{\mu_1\left(\pi^{-1}(y)\right)p+\mu_2\left(\pi^{-1}(y)\right)(1-p)}\]

It follows that for any $q \in (0,1)$

\[q\left(\mu_1\mid \pi^{-1}(y)\right)+(1-q)\left(\mu_2\mid \pi^{-1}(y)\right) \in A\]

We got that $A$ is convex and therefore $N(y)=\overline{A}$ is also convex.
\end{proof}

\begin{samepage}
\begin{proposition}
\label{prp:four_factors}

Consider $W_{1,2,3,4}$ Polish spaces, $X:=\prod_{i=1}^4 W_i$, $Y:=\prod_{i=1}^3 W_i$ and $Z = W_1 \times W_2$. Let $\pi^Y: X \rightarrow Y$, $\pi^Z: X \rightarrow Z$ and $\pi^W: X \rightarrow W_1$ be the projection mappings, $\mu \in \PM(X)$ and $K: Z \M W_3$. Assume that $\pi^Y_* \mu = \pi^Z_*\mu \ltimes K$. Then, for $\pi^W_* \mu$-almost any $w \in W_1$

\begin{equation}
\pi^Y_* \left(\mu \mid \pi^W\right)(w) = \pi^Z_* \left(\mu \mid \pi^W\right)(w) \ltimes K
\end{equation}

\end{proposition}
\end{samepage}

\begin{proof}

We know that $\mu = \left(\mu \mid \pi^W\right)_* \pi^W_* \mu$. It follows that

\[\pi^Y_*\mu = \pi^Y_* \left(\mu \mid \pi^W\right)_* \pi^W_* \mu = \pi^Z_*\left(\mu \mid \pi^W\right)_* \pi^W_* \mu \ltimes K\]

Define $K': Z \M Y$ by $K'(z):=\delta_z \times K(z)$ and denote composition of Markov kernels by $\circ$. Since pushforward commutes with composition, we have

\[\pi^Y_*\mu = \left(\pi^Y \circ\mu \mid \pi^W\right)_* \pi^W_* \mu = \left(K'  \circ \pi^Z \circ \mu \mid \pi^W\right)_* \pi^W_* \mu\]

Let $\pi^{YW}: Y \rightarrow W$ be the projection mapping and denote $\nu = \pi^Y_* \mu$. We get

\[\nu = \left(\pi^Y \circ\mu \mid \pi^W\right)_* \pi^{YW}_* \nu = \left(K'  \circ \pi^Z \circ \mu \mid \pi^W\right)_* \pi^{YW}_* \nu\]

We also know that

\[\Sp{\pi^{YW}_* \nu \ltimes \mu \mid \pi^W} = \Sp{\pi^{W}_* \mu \ltimes \mu \mid \pi^W \subseteq \Gr{\pi^W}}\]
 
As easy to see, the above implies

\[\Sp{\pi^{YW}_* \nu \ltimes \left(\pi^Y \circ\mu \mid \pi^W\right)} \subseteq \Gr{\pi^{YW}}\]

\[\Sp{\pi^{YW}_* \nu \ltimes \left(K'  \circ \pi^Z \circ \mu \mid \pi^W\right)} \subseteq \Gr{\pi^{YW}}\]

By the uniqueness of regular conditional probability, we conclude that for $\pi^W_* \mu$-almost any $w \in W_1$

\[\left(\nu \mid \pi^{YW}\right)(w)= \left(\pi^Y \circ\mu \mid \pi^W\right)(w) = \left(K'  \circ \pi^Z \circ \mu \mid \pi^W\right)(w)\]

This immediately implies the desired result.
\end{proof}

\begin{proof}[Proof of Example~\ref{exm:update_incomplete_kernels}]
Fix $\rho$ a metrization of $\OO$. The condition $\Sp{\mu} \subseteq y\OO$ is closed because it is equivalent to the vanishing of the continuous function $\E_{x\sim\mu}\left[\rho\left(x,y\OO\right)\right]$. The other condition in the definition of $M^K_n$ is closed because pushforward by a continuous mapping and semidirect product with a Feller continuous kernel are continuous in the weak topology. This gives us condition~\ref{con:def__update_incomplete__clos}. Condition~\ref{con:def__update_incomplete__conv} holds since the convex combination of measures supported on $y\OO$ is supported on $y\OO$ and since pushforward and semidirect product commute with convex combinations.

Consider any $\mu \in M^K$. We know that for any $m \geq n$, $\PO_{m+1*}\mu=\PO_{m*}\mu \ltimes K_m$. Applying Proposition~\ref{prp:four_factors}, this implies that for $\PO_{n*}\mu$-almost any $y\in\Ob^n$

\[\PO_{m+1*}\left(\mu \mid \PO_n\right)(y)=\PO_{m*}\left(\mu \mid \PO_n\right)(y) \ltimes K_m\]

We conclude that $\left(\mu \mid \PO_n\right)(y) \in M^K_n(y)$, proving condition~\ref{con:def__update_incomplete__cond}.
\end{proof}

\begin{samepage}
\begin{proposition}
\label{prp:differentiation}

Consider $X$ a compact Polish space, $d \in \Nats$, $Y$ a compact subset of $\Reals^d$, $\pi: X \rightarrow Y$ continuous and $\mu \in \PM\left(X\right)$. Then, for $\pi_*\mu$-almost any $y \in Y$

\begin{equation}
\lim_{r \rightarrow 0}{\mu \mid \pi^{-1}\left(\B_r\left(y\right)\right) = \left(\mu \mid \pi\right)\left(y\right)}
\end{equation}

\end{proposition}
\end{samepage}

\begin{proof}

Let $\Sq{f_k \in C\left(X\right)}{k}$ be dense in $C\left(X\right)$. For any $y \in Y$ and $r > 0$ we denote $\chi_{yr}: Y \rightarrow \{0,1\}$ the characteristic function of $\B_r\left(y\right)$ and $\chi^\pi_{yr}: X \rightarrow \{0,1\}$ the characteristic function of $\pi^{-1}\left(\B_r\left(y\right)\right)$. For any $k \in \Nats$ and $y_0 \in \Sp \pi_* \mu$ we have $\mu\left(\pi^{-1}\left(\B_r\left(y_0\right)\right)\right) = \pi_*\mu\left(\B_r\left(y_0\right)\right) > 0$ and

$$\E_{\mu}\left[f_k \mid \pi^{-1}\left(\B_r\left(y_0\right)\right)\right] = \frac{\E_{\mu}\left[\chi^\pi_{y_0r} f_k\right]}{\mu\left(\pi^{-1}\left(\B_r\left(y_0\right)\right)\right)} = \frac{\E_{y \sim \pi_* \mu}\left[\E_{\mu}\left[\chi^\pi_{y_0r} f_k \mid \pi^{-1}\left(y\right)\right]\right]}{\pi_*\mu\left(\B_r\left(y_0\right)\right)}$$

$$\E_{\mu}\left[f_k \mid \pi^{-1}\left(\B_r\left(y_0\right)\right)\right] = \frac{\E_{y \sim \pi_* \mu}\left[\chi_{y_0r} \E_{\mu}\left[f_k \mid \pi^{-1}\left(y\right)\right]\right]}{\pi_*\mu\left(\B_r\left(y_0\right)\right)}$$

In particular, the above holds for $\pi_* \mu$-almost any $y_0 \in Y$. Applying the Lebesgue differentiation theorem, we conclude that there is $A \subseteq Y$ s.t. $\pi_*\mu\left(A\right) = 1$ and for any $y \in A$

$$\forall k \in \Nats: \lim_{r \rightarrow 0} \E_{\mu}\left[f_k \mid \pi^{-1}\left(\B_r\left(y\right)\right)\right] = \E_{\mu}\left[f_k \mid \pi^{-1}\left(y\right)\right]$$

Now consider any $f \in C\left(X\right)$. For any $\epsilon > 0$, there is $k \in \Nats$ s.t. $\N{f-f_k} < \epsilon$ and therefore, for any $y \in A$

$$\limsup_{r \rightarrow 0} \E_{\mu}\left[f \mid \pi^{-1}\left(\B_r\left(y\right)\right)\right] \leq \limsup_{r \rightarrow 0} \E_{\mu}\left[f_k \mid \pi^{-1}\left(\B_r\left(y\right)\right)\right] + \epsilon = \E_{\mu}\left[f_k \mid \pi^{-1}\left(y\right)\right] + \epsilon  $$

$$\limsup_{r \rightarrow 0} \E_{\mu}\left[f \mid \pi^{-1}\left(\B_r\left(y\right)\right)\right] \leq \E_{\mu}\left[f \mid \pi^{-1}\left(y\right)\right] + 2\epsilon$$

Similarly

$$\liminf_{r \rightarrow 0} \E_{\mu}\left[f \mid \pi^{-1}\left(\B_r\left(y\right)\right)\right] \geq \E_{\mu}\left[f \mid \pi^{-1}\left(y\right)\right] - 2\epsilon$$

Taking $\epsilon$ to 0, we conclude that

$$\lim_{r \rightarrow 0} \E_{\mu}\left[f \mid \pi^{-1}\left(\B_r\left(y\right)\right)\right] = \E_{\mu}\left[f \mid \pi^{-1}\left(y\right)\right]$$

$$\lim_{r \rightarrow 0} \mu \mid \pi^{-1}\left(\B_r\left(y\right)\right) = \left(\mu \mid \pi\right)\left(y\right)$$
\end{proof}

\begin{proof}[Proof of Example~\ref{exm:update_incomplete_euclid}]
There is a unique minimum since all three defining properties of $N$ are closed w.r.t. arbitrary set intersections. To see $N$ is a regular upper bound, note that conditions $\ref{con:def__update_incomplete__clos}$ and $\ref{con:def__update_incomplete__conv}$ are part of the definition whereas condition $\ref{con:def__update_incomplete__cond}$ follows immediately from Proposition~\ref{prp:differentiation}.
\end{proof}

\section*{Acknowledgments}

This work was supported by the Machine Intelligence Research Institute in Berkeley, California.

We wish to thank Michael Greinecker for pointing out Theorem~\ref{thm:selection} to us.

\bibliographystyle{unsrt}
\bibliography{Incomplete_Models}

\end{document}